\documentclass[11pt]{article}

\usepackage{hyperref}       
\usepackage{url}            
\usepackage{booktabs}       
\usepackage{amsfonts}       
\usepackage{nicefrac}       
\usepackage{microtype}      

\usepackage{authblk} 
\usepackage{amsmath,amsthm,amssymb,dsfont}
\usepackage{bbm,multirow}
\allowdisplaybreaks
\usepackage{tikz}
\usepackage{times}
\usepackage{fullpage}
\usepackage{enumitem}
\usepackage{multirow}
\usepackage{bm}
\usepackage{commath}
\usepackage{pdfpages}
\usepackage{color, colortbl}

\usepackage{verbatim}
\usepackage{multicol}
\usepackage{caption}
\usepackage{breqn}

\usepackage[utf8]{inputenc} 
\usepackage[T5]{fontenc} 

\usepackage{algorithm}
\usepackage{algorithmic}

\newtheorem{theorem}{Theorem}
\newtheorem*{theorem*}{Theorem}
\newtheorem{lemma}{Lemma}
\newtheorem*{lemma*}{Lemma}

\newtheorem{claim}{Claim}
\newtheorem{definition}{Definition}

\theoremstyle{definition}
\newtheorem{remark}{Remark}

\newenvironment{claimproof}{\noindent\emph{Proof of claim.}}{\hfill$\qed$}

\newcommand{\cmt}[1]{}

\newcommand{\vect}[1]{\ensuremath{\bm{#1}}}
\newcommand{\E}{\ensuremath{\mathbb{E}}}

\begin{document}

\title{Online Non-Monotone DR-Submodular Maximization\thanks{This work is supported by the ANR project OATA n\textsuperscript{o} ANR-15-CE40-0015-01}} 
\author{
  Nguyễn Kim Thắng\thanks{kimthang.nguyen@univ-evry.fr} }  
\author{
  Abhinav Srivastav\thanks{abhinavsriva@gmail.com}}
\affil{IBISC, Univ. Evry, University Paris-Saclay, France } 
\date{}
\maketitle

\begin{abstract}

In this paper, we study fundamental problems of maximizing DR-submodular continuous functions that have real-world applications in the domain of 
machine learning, economics, operations research and communication systems. It captures a subclass of non-convex optimization that provides both 
theoretical and practical guarantees. 
Here, we focus on 
minimizing regret for online arriving non-monotone DR-submodular functions 
over different types of convex sets: hypercube, down-closed and general convex sets.

First, we present an online algorithm that achieves a $1/e$-approximation ratio with 
the regret of $O(T^{2/3})$ for maximizing DR-submodular functions over any down-closed convex set. 
Note that, the approximation ratio of $1/e$ matches the best-known guarantee for the offline version of the problem. 
Moreover, when the convex set is the hypercube, we propose a tight 1/2-approximation algorithm with regret bound of 
$O(\sqrt{T})$.
Next, we give an online algorithm that achieves an approximation guarantee (depending on the search space) for the problem of maximizing non-monotone continuous DR-submodular functions over a \emph{general} convex set (not necessarily down-closed). To best of our knowledge, no prior algorithm with approximation guarantee was known for non-monotone DR-submodular 
maximization in the online setting. 
Finally we run experiments to verify the performance of our algorithms on problems arising in machine learning domain with the real-world datasets.  
\end{abstract}


\section{Introduction}

Continuous DR-submodular optimization is a subclass of non-convex optimization that is an upcoming frontier in machine learning.
Roughly speaking, a differentiable non-negative bounded function $F: [0,1]^{n} \rightarrow [0,1]$ is \emph{DR-submodular} if 
$\nabla F(\vect{x}) \geq \nabla F(\vect{y})$ for all $\vect{x}, \vect{y} \in [0,1]^{n}$ where $x_{i} \leq y_{i}$ 
for every $1 \leq i \leq n$. (Note that, w.l.o.g. after normalization, assume that $F$ has values in $[0,1]$.) 
Intuitively, continuous DR-submodular functions 
represent the diminishing returns property or the economy of scale in continuous domains. 
DR-submodularity has been of great interest
\cite{BianBuhmann18:Optimal-DR-Submodular,BianLevy17:Non-monotone-Continuous,ChenHassani18:Online-continuous, HassaniSoltanolkotabi17:Gradient-methods, NiazadehRoughgarden18:Optimal-Algorithms,StaibJegelka17:Robust-Budget}.
Many problems arising in machine learning and statistics, 
such as Non-definite Quadratic Programming \cite{ItoFujimaki16:Large-scale-price}, 
Determinantal Point Processes \cite{KuleszaTaskar12:Determinantal-point}, 
log-submodular models \cite{DjolongaKrause14:From-MAP-to-marginals:}, to name a few, have been modelled using the notion of continuous DR-submodular functions.


In the past decade, online computational framework has been quite successful for tackling a wide variety of challenging problems and capturing many real-world problems with uncertainty.
In this computational framework, we focus on the model of online learning. In online learning, at any time step, given a history of actions and a set of associated reward functions, online algorithm first chooses an action from a set of feasible actions; then, an adversary subsequently selects a reward function.  The objective is to perform as good as the best fixed action in hindsight. 
This setting have been extensively explored in the literature, especially in context of convex functions~\cite{Hazanothers16:Introduction-to-online}.  

In fact, several algorithms with theoretical approximation guarantees are known for maximizing (offline and online) DR-submodular functions. However, these guarantees hold under the assumptions that are based on the monotonicity of functions coupled with the structure of convex sets such as unconstrained hypercube and down-closed. Though, a majority of real-world problems can be formulated as \emph{non-monotone} DR-submodular functions over convex sets that might not be necessarily down-closed. 
For example, the problems of Determinantal Point Processes, log-submodular models, etc can be viewed as non-monotone DR-submodular maximization problems.
Besides, general convex sets include conic convex sets,  up-closed convex sets, mixture of covering and packing linear constraints, etc. which appear in many applications. Among others, conic convex sets play a crucial role in convex optimization. Conic programming~\cite{BoydVandenberghe04:Convex-optimization} --- an important subfield of convex optimization --- consists of 
optimizing an objective function over a conic convex set. Conic programming reduces to linear programming and semi-definite programming when the objective function is linear and the convex cones are 
the positive orthant $\mathbb{R}^{n}_{+}$  
and positive semidefinite matrices $\mathbb{S}^{n}_{+}$, respectively. Optimizing non-monotone DR-submodular functions over a (bounded) conic convex set (not necessarily downward-closed) is an interesting and important problem both in theory and in practice.  To best of our knowledge, no prior work has been done on maximizing DR-submodular functions over conic sets in online setting. The limit of current theory \cite{BianBuhmann18:Optimal-DR-Submodular,BianLevy17:Non-monotone-Continuous,ChenHassani18:Online-continuous,HassaniSoltanolkotabi17:Gradient-methods,NiazadehRoughgarden18:Optimal-Algorithms, Zhang:2019} motivates us to develop online algorithms for non-monotone functions. 

In this work, we explore the \emph {online} problem of maximizing \emph{non-monotone} DR-submodular functions over a hypercube, down-closed\footnote{$\mathcal{K}$ is \emph{down-closed} if for every $\vect{z} \in \mathcal{K}$ and $\vect{y} \leq \vect{z}$
then $\vect{y} \in \mathcal{K}$} and over a \emph{general} convex sets. 
Formally, we consider the following setting for DR-submodular maximization: given a convex domain $\mathcal{K} \subseteq [0,1]^n$ in advance, at each time step $t = 1,2, \ldots, T$, the online algorithm first selects a vector $\vect{x}^t \in \mathcal{K}$. Subsequently, the adversary reveals a non-monotone continuous DR-submodular function $F^{t}$ 
and the algorithm receives a reward of $F^t(\vect{x}^t)$. The objective is also to maximize the total reward. 
We say that an algorithm achieves a $(r, R(T))$-\emph{regret} if  
\begin{align*}
\sum_{t=1}^T F^t(\vect{x}^t) 
		 	&\geq r \cdot \max_{\vect{x} \in \cal K} \sum_{t=1}^T F^t(\vect{x}) - R(T) 
\end{align*}
In other words, $r$ is the approximation ratio that measures 
the quality of the online algorithm compared to the best fixed solution in hindsight and $R(T)$ represents the regret in the classic terms. 
Equivalently, we say that the algorithm has $r$-\emph{regret} at most $R(T)$. 
Our goal is to design online algorithms with $(r, R(T))$-regret where $0 < r \leq 1$ is as large as possible, 
and $R(T)$ is sub-linear in $T$, i.e., $R(T) = o(T)$.

\subsection{Our contributions and techniques}

In this paper, we provide algorithms with performance guarantees for the problem over the hypercube $[0,1]^{n}$, 
over down-closed and general convex sets.  
Our contributions and techniques are summarized as follows (also see Table~\ref{table:DR-submodular}, the entries in the red correspond to our contribution). 

\newcommand{\STAB}[1]{\begin{tabular}{@{}c@{}}#1\end{tabular}}
\renewcommand{\arraystretch}{1.5}

\begin{table}[th]
  \small
  \centering
  \begin{tabular}{|c|c||c|c|}
    \hline
     \multicolumn{2}{|c||}{} & Monotone & Non-monotone\\
    \hline
    \hline
    \multirow{3}{*}{\STAB{\rotatebox[origin=c]{90}{Offline}}}& hypercube &  & $\frac{1}{2}$-approx \cite{BianBuhmann18:Optimal-DR-Submodular,NiazadehRoughgarden18:Optimal-Algorithms} \\
    		\cline{2-4}
    		 & down-closed 
		 & $\left( 1 - \frac{1}{e} \right)$-approx \cite{BianMirzasoleiman17:Guaranteed-Non-convex}  
		 & $\left(\frac{1}{e} \right)$-approx \cite{BianLevy17:Non-monotone-Continuous}    \\
		 \cline{2-4}
		 & general  & $\left(1 - \frac{1}{e} \right)$-approx \cite{MokhtariHassani18:Conditional-Gradient}
		 &$\left(  \frac{1- \min_{\vect{x} \in \mathcal{K}} \|\vect{x}\|_{\infty} }{3\sqrt{3}} \right)$-approx \cite{DurrThang20:Non-monotone-DR-submodular}   \\
    \hline
    \hline
     \multirow{3}{*}{\STAB{\rotatebox[origin=c]{90}{Online}}}& hypercube & & \color{red}{$\left( 1/2, O\bigl(\sqrt{T}\bigr) \right)$}  \\
    		\cline{2-4}
    		 & down-closed  &  & 
		 	\color{red}{$\left(\frac{1}{e}, O\bigl(T^{2/3}\bigr) \right)$}\\
		 \cline{2-4}
		 & general  &  $\left( 1 - \frac{1}{e}, O\bigl(\sqrt{T}\bigr) \right)$ \cite{ChenHassani18:Online-continuous}  & \color{red}{$\left(  \frac{1- \min_{\vect{x} \in \mathcal{K}} \|\vect{x}\|_{\infty} }{3\sqrt{3}}, O\bigl(\frac{T}{\ln T}\bigr) \right)$}    \\
    \hline
  \end{tabular}
  \vspace{0.2cm}
  \caption{Summary of results on DR-submodular maximization.
  	The entries represent the best-known $(r,R(T))$-regret; our results are shown in red. 
	The guarantees in blank cells can be deduced from a more general one.
	}
  \label{table:DR-submodular}
\end{table}

\subsubsection{Maximizing online non-monotone DR-submodular functions over down-closed convex sets} 

We present an online algorithm that achieves 
$\bigl(1/e, O(T^{2/3})\bigr)$-regret in expectation
where $T$ is number of time steps.  
Our algorithm is inspired by the Meta-Frank-Wolfe algorithm introduced by Chen et al.~\cite{ChenHassani18:Online-continuous} 
for \emph{monotone} DR-submodular maximization. 
Their meta Frank-Wolfe algorithm combines the framework of meta-actions proposed in \cite{StreeterGolovin09:An-online-algorithm} with a variant of the Frank-Wolfe proposed in \cite{BianMirzasoleiman17:Guaranteed-Non-convex} for maximizing monotone DR-submodular functions.
Informally, at every time $t$, our algorithm starts from the origin $\vect{0}$ ($\vect{0} \in \mathcal{K}$ since $\mathcal{K}$ is down-closed) and executes $L$ steps of the Frank-Wolfe algorithm 
where every update vector at iteration $1 \leq \ell \leq L$ is constructed by 
combining the output of an optimization oracle $\mathcal{E}_{\ell}$
and the current vector so that we can exploit the \emph{concavity property in positive directions} of DR-submodular functions. 
The solution $\vect{x}^{t}$ is produced at the end of the $L$-th step.
Subsequently, after observing the function $F^{t}$, the algorithm subtly defines a vector $\vect{d}^{t}_{\ell}$ and feedbacks  
$\langle \cdot, \vect{d}^t_{\ell} \rangle$ as the reward function at time $t$ to the oracle $\mathcal{E}_{\ell}$ for $1 \leq \ell \leq L$.
The most important distinguishing point of our algorithm compared 
to that in \cite{ChenHassani18:Online-continuous} relies on the oracles. 
In their algorithm, they used linear oracles which are appropriate for maximizing monotone DR-submodular functions.
However, it is not clear whether \emph{linear or even convex oracles} are strong enough to achieve a constant approximation for non-monotone functions. 

While aiming for a $1/e$-approximation --- the best known approximation in the offline setting --- we consider 
the following online non-convex problem, referred throughout this paper as \emph{online vee learning} problem.
At each time $t$, the online algorithm knows 
$\vect{c}^t \in \mathcal{K}$ in advance and selects a vector $\vect{x}^t \in \mathcal{K}$. 
Subsequently, the adversary reveals a vector $\vect{a}^t \in \mathbb{R}^{n}$ 
and the algorithm receives a reward $\langle \vect{a}^t, \vect{c}^t \vee \vect{x}^t \rangle$.
Given two vectors $\vect{x}$ and $\vect{y}$, the vector $\vect{x} \vee \vect{y}$ has the $i^{th}$ coordinate   
$(x \vee y)_{i} = \max\{x(i), y(i)\}$.
The goal is to maximize the total reward over a time horizon $T$.


In order to bypass the non-convexity of the online vee learning problem, we propose the following methodology. 
Unlike dimension reduction techniques that aim at reducing the dimension of the search space, 
we \emph{lift} the search space and the functions to a higher dimension so as to benefit from their properties in this new space.  
Concretely, at a high level, given a convex set $\mathcal{K}$, we define a ``sufficiently dense'' lattice $\mathcal{L}$ in $[0,1]^{n}$ such that every point in $\mathcal{K}$ can be approximated by a point in the lattice. The term ``sufficiently dense'' corresponds to the fact that the values of any Lipschitz function can be approximated by the function values at lattice points. As the next step, we lift all lattice points in $\mathcal{K}$ to a high dimension space so that they form a subset of vertices (corners points) $\widetilde{\mathcal{K}}$ of the unit hypercube in the newly defined space. Interestingly, the reward function 
$\langle \vect{a}^t, \vect{c}^t \vee \vect{x}^t \rangle$ can be transformed to a linear function 
in the new space. Hence, our algorithm for the online vee problem consists of  
updating at every time step a solution in the high-dimension space using a linear oracle and projecting the solution back to the original space. 

Once the solution is projected back to original space, we
construct appropriate update vectors for the original DR-submodular maximization and feedback rewards for the online vee oracle.
Exploiting the underlying properties of DR-submodularity, we show that our algorithm achieves the regret bound of $(1/e, O(T^{2/3}))$.


\subsubsection{Maximizing online non-monotone DR-submodular functions over the hypercube}

Next, we consider a particular down-closed convex set, the hypercube $[0,1]^{n}$, that has 
been widely studied in the context of submodular maximization. In offline setting, several algorithms with tight approximation 
guarantee of $1/2$ are known \cite{BianBuhmann18:Optimal-DR-Submodular,NiazadehRoughgarden18:Optimal-Algorithms}; however it is not clear how to extend those algorithms to the online setting. 
Besides, in the context of discrete setting, Roughgarden and Wang~\cite{RoughgardenWang18:An-Optimal-Learning} 
recently gave 1/2-approximation algorithm for the problem of maximizing online (discrete) submodular functions.
In the light of this, an intriguing open question is that if one can achieve $1/2$-approximation for DR-submodular maximization in the continuous setting. Note that several properties used for achieving $1/2$-approximation solutions in the discrete setting or 
in the offline continuous setting do not necessarily hold in the online settings.  

Building upon our salient ideas from previous algorithm,  we present an online algorithm that achieves $1/2$-approximation for DR-submodular maxmization over the hypercube. Essentially, we reduce the DR-submodular maximization over $[0,1]^{n}$ to 
the problem of maximizing  discrete submodular functions over vertices of the hypercube in a higher dimensional space. We do it also by \emph{lifting} points of a lattice in $[0,1]^{n}$ to vertices in $\{0,1\}^{n \times \log T}$ and prove that this operation transforms DR-submodular functions to submodular set-functions.
Applying algorithm of \cite{RoughgardenWang18:An-Optimal-Learning} and performing an independent rounding, 
we provide an algorithm with the tight approximation ratio of $1/2$ while maintaining the $1/2$-regret bound of $O(\sqrt{T})$.            

\subsubsection{Maximizing online non-monotone DR-submodular functions over general convex sets}
For general convex sets, we prove that the Meta-Frank-Wolfe algorithm with adapted step-sizes guarantees 
$\left(\frac{1-\min_{\vect{x} \in \mathcal{K}} \|\vect{x}\|_{\infty} }{3\sqrt{3}}, O\bigl(\frac{T}{\ln T}\bigr) \right)$-regret 
in expectation. Notably, if the domain $\cal K$ 
contains $\vect{0}$ (as in the case where $\cal K$ is the intersection of a cone and the hypercube $[0,1]^{n}$) then  
the algorithm guarantees in expectation a $\bigl( \frac{1}{3\sqrt{3}}, O\bigl( \frac{T}{\log T} \bigr) \bigr)$-regret.
To the best of our knowledge, this is the \emph{first} constant-approximation algorithm for the problem of  
online non-monotone DR-submodular maximization over a non-trivial convex set that goes beyond down-closed sets.
Remark that the quality of the solution, specifically the approximation ratio, depends on the initial solution 
$\vect{x}_{0}$. This confirms an observation in various contexts that initialization plays an important role in non-convex optimization.

\subsection{Related Work}

In this section, we give a summary on best-known results on 
DR-submodular maximization. 
The domain has been investigated more extensively 
in recent years due to its numerous applications in the field of statistics and machine learning, for example 
active learning \cite{GolovinKrause11:Adaptive-submodularity:}, 
viral makerting \cite{KempeKleinberg03:Maximizing-the-spread}, 
network monitoring \cite{Gomez-RodriguezLeskovec10:Inferring-networks}, 
document summarization \cite{LinBilmes11:A-class-of-submodular},
crowd teaching \cite{SinglaBogunovic14:Near-Optimally-Teaching}, 
feature selection \cite{ElenbergKhanna18:Restricted-strong}, 
deep neural networks \cite{ElenbergDimakis17:Streaming-weak}, 
diversity models \cite{DjolongaTschiatschek16:Variational-inference} and
recommender systems \cite{GuilloryBilmes11:Simultaneous-Learning}.

\paragraph{Offline setting.}
Bian et al.~\cite{BianMirzasoleiman17:Guaranteed-Non-convex} considered the problem of maximizing monotone DR-functions subject to a down-closed convex set and showed that the greedy method proposed by \cite{CalinescuChekuri11:Maximizing-a-monotone}, 
a variant of well-known Frank-Wolfe algorithm in convex optimization, guarantees a $(1-1/e)$-approximation.
However, it has been observed by \cite{HassaniSoltanolkotabi17:Gradient-methods} 
that the greedy method is not robust in stochastic settings
(where only unbiased estimates of gradients are available). Subsequently, 
Hassani et al.~\cite{MokhtariHassani18:Conditional-Gradient} proposed a $(1-1/e)$-approximation algorithm 
for maximizing monotone DR-submodular functions over general convex sets in stochastic settings
by a new variance reduction technique. 


The problem of maximizing non-monotone DR-submodular functions is much harder. 
Bian et al.~\cite{BianBuhmann18:Optimal-DR-Submodular} and Niazadeh et al.~\cite{NiazadehRoughgarden18:Optimal-Algorithms}
have independently presented algorithms with the same approximation guarantee of $(1/2)$ for the problem of 
non-monotone DR-submodular maximization over the hypercube ($\mathcal{K} = [0,1]^{n}$).
These algorithms are inspired by the bi-greedy algorithms in \cite{BuchbinderFeldman15:A-tight-linear,BuchbinderFeldman18:Deterministic-algorithms}. 
Bian et al.~\cite{BianLevy17:Non-monotone-Continuous} made a further step by providing a $(1/e)$-approximation algorithm where the convex sets are down-closed. 
Mokhtari et al.~\cite{MokhtariHassani18:Conditional-Gradient} also presented an  algorithm that achieve $(1/e)$ for non-monotone continuous DR-submodular function over a down-closed convex domain that uses only stochastic gradient estimates.
Remark that when aiming for approximation algorithms (in polynomial time), the restriction to down-closed polytopes is unavoidable.
Specifically, Vondrak~\cite{Vondrak13:Symmetry-and-approximability} proved that any algorithm for the problem over a non-down-closed set that guarantees a constant approximation must require in general exponentially many value queries to the function. 
Recently,  Durr et al.~\cite{DurrThang20:Non-monotone-DR-submodular} presented an algorithm that achieves $\bigl((1-\min_{\vect{x} \in \mathcal{K}} \|\vect{x}\|_{\infty})/3\sqrt{3} \bigr)$-approximation for non-monotone DR-submodular function over general convex sets that are not necessarily down-closed.

\paragraph{Online setting.}
The results for online DR-submodular maximization are known \emph{only} for monotone functions. 
Chen et al.~\cite{ChenHassani18:Online-continuous} considered the online problem of maximizing monotone DR-submodular functions over general convex sets and provided an algorithm that achieves a $\bigl(1-1/e, O(\sqrt{T})\bigr)$-regret. 
Subsequently, Zhang et al.~\cite{Zhang:2019} presented an algorithm that reduces the number of per-function gradient evaluations from $T^{3/2}$ in \cite{ChenHassani18:Online-continuous} to $1$ and achieves the same approximation ratio of $(1-1/e)$. Leveraging the idea of one gradient per iteration, they presented a bandit algorithm for maximizing monotone DR-submodular function over a general convex set that achieves in expectation $(1-1/e)$ approximation ratio with regret $O(T^{8/9})$.
Note that in the discrete setting, 
\cite{RoughgardenWang18:An-Optimal-Learning} studied the non-monotone (discrete) submodular maximization over the hypercube and gave an algorithm which guarantees the tight $(1/2)$-approximation and $O(\sqrt{T})$ regret.  
%



\section{Preliminaries and Notations }		\label{sec:pre}

Throughout the paper, we use bold face letters, e.g., $\vect{x}, \vect{y}$ to represent vectors. 
For every $S \subseteq \{1, 2, \ldots, n \}$, vector $\vect{1}_S$ is the $n$-dim vector whose $i^{th}$ coordinate equals to 1 if $i \in S$ and $0$ otherwise. 
Given two $n$-dimensional vectors $\vect{x}, \vect{y}$, we say that $\vect{x} \leq \vect{y}$ iff $x_{i} \leq y_{i}$ for all $1 \leq i \leq n$.
Additionally, we denote by $\vect{x} \vee \vect{y}$, their coordinate-wise maximum vector such that $(x \vee y)_{i} = \max\{x_{i}, y_{i}\}$.
Moreover, the symbol $\circ$ represents the element-wise multiplication that is,  given two vectors $\vect{x}, \vect{y}$, vector $\vect{x} \circ \vect{y}$ is the such that the $i$-th coordinate $(\vect{x} \circ \vect{y})_{i} = x_{i} y_{i}$.
The scalar product $\langle \vect{x}, \vect{y} \rangle = \sum_{i} x_{i}y_{i}$
and the norm $\|\vect{x}\| = \langle \vect{x}, \vect{x} \rangle^{1/2}$. 
In the paper, we assume that  $\mathcal{K} \subseteq [0,1]^{n}$. 
We say that 
$\mathcal{K}$ is the \emph{hypercube} if 
$\mathcal{K} = [0,1]^{n}$;
$\mathcal{K}$ is \emph{down-closed} if for every $\vect{z} \in \mathcal{K}$ and $\vect{y} \leq \vect{z}$
then $\vect{y} \in \mathcal{K}$; and
$\mathcal{K}$ is \emph{general} if $\mathcal{K}$ is a convex subset of $[0,1]^n$ without any special property.

\begin{definition}
A function $F: [0,1]^{n} \rightarrow \mathbb{R}^+ \cup \{0\} $ is \emph{diminishing-return (DR) submodular} if for all vector 
$\vect{x} \geq \vect{y} \in [0,1]^{n}$, any basis vector $\vect{e}_{i} = (0,\ldots,0,1,0,\ldots,0)$ and any constant 
$\alpha > 0$ such that 
$\vect{x} + \alpha \vect{e}_{i} \in [0,1]^{n}$, $\vect{y} + \alpha \vect{e}_{i} \in [0,1]^{n}$, 
it holds that
\begin{align}	\label{def:DR-sub}
	F(\vect{x} + \alpha \vect{e}_{i}) - F(\vect{x}) \leq F(\vect{y} + \alpha \vect{e}_{i}) - F(\vect{y}).
\end{align}
Note that if function $F$ is differentiable then the diminishing-return (DR) property (\ref{def:DR-sub}) is equivalent to 
\begin{align}	\label{def:DR-sub-diff}
\nabla F(\vect{x}) \leq \nabla F(\vect{y}) \qquad \forall \vect{x} \geq \vect{y} \in [0,1]^{n}.
\end{align}
Moreover, if $F$ is twice-differentiable then 
the DR property is equivalent to all of the entries of its Hessian being non-positive, i.e., 
$\frac{\partial^{2} F}{\partial x_{i} \partial x_{j}} \leq 0$ for all $1 \leq i, j \leq n$.
\end{definition}

Besides, a differentiable function $F: \mathcal{K} \subseteq [0,1]^n \rightarrow \mathbb{R}^+ \cup \{ 0 \} $ is said to be \textit{$\beta$-smooth}  if for any $\vect{x}, \vect{y} \in \mathcal{K}$, the following holds, 
\begin{align}	\label{ineq:smooth1}
F(\vect{y}) \leq F(\vect{x}) + \langle \nabla F(\vect{x}), \vect{y} - \vect{x} \rangle + \frac{\beta}{2} \|\vect{y} - \vect{x}\|^{2}
\end{align}
or equivalently, the gradient is $\beta$-Lipschitz, i.e., 
\begin{align}	\label{ineq:smooth2}
	\norm{\nabla F(\vect{x}) - \nabla F(\vect{y})} \leq \beta \norm{\vect{x} - \vect{y}}.
\end{align}


\paragraph{Properties of DR-submodularity}
In the following, we present a property of DR-submodular functions that are are crucial in our analyses.
The most important property is the concavity in positive directions, i.e., for every $\vect{x}, \vect{y} \in \mathcal{K}$
and $\vect{x} \geq \vect{y}$, it holds that 
$$
F(\vect{x}) \leq F(\vect{y}) + \langle \nabla F(\vect{y}), \vect{x} - \vect{y} \rangle.
$$
Addionally, the following lemmas are useful in our analysis. 

\begin{lemma}[\cite{HassaniSoltanolkotabi17:Gradient-methods}] 	\label{lem:prop2}
For every $\vect{x}, \vect{y} \in \mathcal{K}$ and any DR-submodular function $F$, 
it holds that
$$ \langle \nabla F(\vect{x}), \vect{y} - \vect{x}\rangle \geq F(\vect{x} \vee \vect{y}) + F(\vect{x} \wedge \vect{y}) - 2F(\vect{x}) $$
\end{lemma}

\begin{lemma}[\cite{FeldmanNaor11:A-unified-continuous,ChekuriJayram15:On-multiplicative-weight,BianLevy17:Non-monotone-Continuous}]		\label{lem:obs2}
For every $\vect{x}, \vect{y} \in \mathcal{K}$ and for any DR-submodular function $F: [0,1]^n \rightarrow \mathbb{R}^+ \cup \{0\}$, it holds that $F(\vect{x} \vee \vect{y}) \geq \bigl( 1 - \|\vect{x} \|_{\infty} \bigr) F(\vect{y})$.
\end{lemma}

%

\paragraph{Variance Reduction Technique.}
\label{var-reduce}

Our algorithm and its analysis relies on a recent variance reduction technique proposed by~\cite{MokhtariHassani18:Conditional-Gradient}. 
This theorem has also been used in the context of online monotone submodular maximization \cite{ChenHarshaw18:Projection-Free-Online,Zhang:2019}.

\begin{lemma}[\cite{MokhtariHassani18:Conditional-Gradient}, Lemma 2]  \label{lem:variance-reduction}
Let $\{\vect{a}^t\}_{t=0}^T$ be a sequence of points in $\mathbb{R}^n$. such that $\norm{\vect{a}^t - \vect{a}^{t-1}} \leq C/(t+s)$ for all $1 \leq t \leq T$ with fixed constants $C \geq 0$ and $s \geq 3$. Let $\{ \tilde{\vect{a}} \}_{t =1}^T$ be a sequence of random variables such that $\mathbb{E} [\tilde{\vect{a}}|\mathcal{H}^{t-1}] = \vect{a}^t$ and $\mathbb{E} \left[ \norm{\tilde{\vect{a}}^t - \vect{a}^t}^2| \mathcal{H}^{t-1}\right] \leq \sigma^2$ for every $t \geq 0$, where $\mathcal{H}^{t-1}$ is the history up to $t-1$.
Let $\{ \vect{d}^t\}_{t=0}^T$ be a sequence of random variables where $\vect{d}^0$ is fixed and subsequent $\vect{d}^t$ are obtained by the recurrence 
$$ \vect{d}^t = (1 - \rho^t) \vect{d}^{t-1} + \rho^t \tilde{\vect{a}}^t$$ 
with $\rho^t = \frac{2}{(t+s)^{2/3}}$. Then we have 
$$ \mathbb{E} \left[ \norm{\vect{a}^t - \vect{d}^t}^2\right] \leq \frac{Q}{(t+s+1)^{2/3}},$$
where $Q = \max \{\norm{\vect{a}^0 - \vect{d}^0}^2 (s+1)^{2/3}, 4\sigma^2 + 3C^2/2 \}$. 
\end{lemma}

%


\section{Online DR-Submodular Maximization over Down-Closed Convex Sets}
\label{sec:down-closed}

\subsection{Online Vee Learning} 
\label{sec:vee}

In this section, we give an algorithm for an online problem that will be the main building block 
in the design of algorithms for online DR-submodular maximization over down-closed convex sets.
In the online vee learning problem, given a down-closed convex set $\mathcal{K}$, 
at every time $t$, the online algorithm receives $\vect{c}^t \in {\cal K}$ at the beginning of the step and needs to choose a vector 
$\vect{x}^t \in \mathcal{K}$. Subsequently, the adversary reveals a vector $\vect{a}^t \in \mathbb{R}^{n}$  
and the algorithm receives a reward $\langle \vect{a}^t, \vect{c}^t \vee \vect{x}^t \rangle$. 
The goal is to maximize the total reward $\sum_{t=1}^T \langle \vect{a}^t, \vect{c}^t \vee \vect{x}^t \rangle$ over a time horizon $T$.

The main issue in the online vee problem is that the reward functions are non-concave. In order to overcome this 
obstacle, we consider a novel approach that consists of discretizing and lifting the corresponding functons to a higher dimensional space. 

\paragraph{Discretization and lifting.}
Let $\mathcal{L}$ be a lattice such that $\mathcal{L} = \{0, \frac{1}{M}, \frac{2}{M}, \ldots,  \frac{\ell}{M}, \ldots, 1\}^{n}$ where $0 \leq \ell \leq M$ for some parameter $M$ (to be defined later). 
Note that each $x_{i} = \frac{\ell}{M}$ for $1 \leq \ell \leq M$ 
can be uniquely represented by a vector $\vect{y}_{i} \in \{0,1\}^{M}$ 
such that $y_{i1} = \ldots = y_{i\ell} = 1$ and $y_{i,\ell+1} = \ldots = y_{i,M} = 0$. 
Based on this observation, we lift the discrete set $\mathcal{K} \cap \mathcal{L}$ to the $(n \times M)$-dim space.
Specifically, define a \emph{lifting} map $m: \mathcal{K} \cap \mathcal{L} \rightarrow \{0,1\}^{n \times M}$
such that each point $(x_{1}, \ldots, x_{n}) \in \mathcal{K} \cap \mathcal{L}$
is mapped to an unique point $(y_{10}, \ldots, y_{1M}, \ldots, y_{n0}, \ldots, y_{nM}) \in \{0,1\}^{n \times M}$
where $y_{i1} = \ldots = y_{i\ell} = 1$ and $y_{i,\ell+1} = \ldots = y_{i,M} = 0$ iff $x_{i} = \frac{\ell}{M}$ for $1 \leq \ell \leq M$.
Define $\widetilde{\mathcal{K}}$ be the set 
$\{ \vect{1}_{X} \in  \{0,1\}^{n \times M}: \vect{1}_{X} = m(\vect{x}) \text{ for some } \vect{x} \in \mathcal{K} \cap \mathcal{L} \}$.
Observe that $\widetilde{\mathcal{K}}$ is a subset of discrete points in $\{0,1\}^{n \times M}$. 
Let $\mathcal{C} := \texttt{conv}(\widetilde{\mathcal{K}})$ be the convex hull of $\widetilde{\mathcal{K}}$.

\paragraph{Algorithm description.}
In our algorithm, at every time $t$, we will output $\vect{x}^{t} \in \mathcal{L} \cap \mathcal{K}$. 
In fact, we will reduce the online vee problem to 
a online linear optimization problem in the $(n \times M)$-dim space. 
Given $\vect{c}^{t} \in {\cal K}$, we round every coordinate $c^{t}_{i}$ 
for $1 \leq i \leq n$ to the largest multiple of $\frac{1}{M}$ which is smaller than $c^{t}_{i}$. In other words, the rounded vector 
$\overline{\vect{c}}^{t}$ has the $i^{th}$-coordinate $\bar{c}^{t}_{i} = \frac{\ell}{M}$ where 
$\frac{\ell}{M} \leq c^{t}_{i} <  \frac{\ell+1}{M}$. Vector $\overline{\vect{c}}^{t} \in \mathcal{L}$ and also $\overline{\vect{c}}^{t} \in \mathcal{K}$ (since $\overline{\vect{c}}^{t} \leq \vect{c}^{t}$ and $\mathcal{K}$ is down-closed). Denote 
$\vect{1}_{C^{t}} = m(\overline{\vect{c}}^{t})$. Besides, for each vector $\vect{a}^{t} \in \mathbb{R}^{n}$, 
define its correspondence $\widetilde{\vect{a}}^{t} \in \mathbb{R}^{n \times M}$ such that $\widetilde{a}^{t}_{i,j} = \frac{1}{M} a^t_i$ for all $1 \leq i \leq n$ and $1 \leq j \leq M$. 
Observe that $\langle \vect{a}^{t}, \overline{\vect{c}}^{t} \rangle = \langle \widetilde{\vect{a}}^{t}, \vect{1}_{C^{t}} \rangle$
(where the second scalar product is taken in the space of dimension $n \times M$).

\begin{algorithm}[h]
\begin{flushleft}
\textbf{Input}:  A convex set $\mathcal{K}$, a time horizon $T$.
\end{flushleft}
\begin{algorithmic}[1]
\STATE Let $\mathcal{L}$ be a lattice in $[0,1]^{n}$ and denote the polytope $\mathcal{C} := \texttt{conv}(\widetilde{\mathcal{K}})$.
\STATE Initialize arbitrarily $\vect{x}^{1} \in \mathcal{K} \cap \mathcal{L}$ and 
	$\vect{y}^{1} = \vect{1}_{X^{1}} = m(\vect{x}^{1}) \in \widetilde{\mathcal{K}}$.
\FOR {$t = 1$ to $T$}	 		
		\STATE Observe $\vect{c}^{t}$, compute $\vect{1}_{C^{t}} = m(\overline{\vect{c}}^{t})$.
		\STATE Play $\vect{x}^{t}$.
		\STATE Observe $\vect{a}^{t}$ (so compute $\widetilde{\vect{a}}^{t}$) 
			and receive the reward $\langle \vect{a}^{t}, \vect{x}^{t} \rangle$.
		\STATE Set $\vect{y}^{t+1} \gets \texttt{update}_{\mathcal{C}}(\vect{y}^{t}; \widetilde{\vect{a}}^{t'} \circ (\vect{1} - \vect{1}_{C^{t'}}): t' \leq t)$ where $\circ$ denotes the element-wise multiplication.
		\STATE Round: $\vect{1}_{X^{t+1}} \gets \texttt{round}(\vect{y}^{t+1})$.
		\STATE Compute $\vect{x}^{t+1} = m^{-1}(\vect{1}_{X^{t+1}})$.
\ENDFOR
\end{algorithmic}
\caption{Online algorithm for vee learning}
\label{algo:vee}
\end{algorithm}

The formal description is given in Algorithm~\ref{algo:vee}.
In the algorithm, we use a procedure \texttt{update}. This procedure takes arguments as the polytope 
$\mathcal{C}$, the current vector $\vect{y}^{t}$, the gradients of previous time steps $\widetilde{\vect{a}}^{t} \circ (\vect{1} - \vect{1}_{C^{t}})$ and outputs the next vector $\vect{y}^{t+1}$. One can use different update strategies, in particular 
the gradient descent or the follow-the-perturbed-leader algorithm (if aiming for a projection-free algorithm).
\begin{align*}
	\vect{y}^{t+1} &\gets \texttt{Proj}_{\mathcal{C}}\left( \vect{y}^{t} - \eta \cdot \widetilde{\vect{a}}^{t} \circ (\vect{1} - \vect{1}_{C^{t}}) \right)
		& \qquad \text{(Gradient Descent)}  \\
	\vect{y}^{t+1} &\gets \arg \min_{\mathcal{C}} \biggl\{ \eta \sum_{t'=1}^{t} \langle \widetilde{\vect{a}}^{t'} \circ (\vect{1} - \vect{1}_{C^{t'}}), \vect{y} \rangle + \vect{n}^{T} \cdot \vect{y} \biggr\}
		& \qquad \text{(Follow-the-perturbed-leader)}
\end{align*}

Besides, in the algorithm, we use additionally a procedure \texttt{round} in order to transform a solution $\vect{y}^{t}$
in the polytope $\mathcal{C}$ of dimension $n \times M$ to an integer point in $\widetilde{\mathcal{K}}$. 
Specifically, given $\vect{y}^{t} \in \texttt{conv}(\widetilde{\mathcal{K}})$, we round $\vect{y}^{t}$ to $\vect{1}_{X^{t}} \in \widetilde{\mathcal{K}}$ using an efficient algorithm given in \cite{MirrokniLeme17:Tight-bounds}
for the approximate Carathéodory's theorem w.r.t the $\ell_{2}$-norm. 
Specifically, given $\vect{y}^{t} \in \mathcal{C}$ and an arbitrary $\epsilon > 0$, 
the algorithm in \cite{MirrokniLeme17:Tight-bounds} returns a set of $k = O(1/\epsilon^{2})$ integer points 
$\vect{1}_{X^{t}_{1}}, \ldots, \vect{1}_{X^{t}_{k}} \in \widetilde{\mathcal{K}}$ such that 
$\| \vect{y}^{t} - \frac{1}{k} \sum_{j=1}^{k}  \vect{1}_{X^{t}_{j}} \| \leq \epsilon$.
Hence, given $\vect{1}_{X^{t}_{1}}, \ldots, \vect{1}_{X^{t}_{k}}$ which have been computed,  
the procedure \texttt{round} simply consists of rounding $\vect{y}^{t}$ to 
$\vect{1}_{X^{t}_{j}}$ with probability $1/k$. Finally, the solution $\vect{x}^{t}$ in the original space is computed 
as $m^{-1}(\vect{1}_{X^{t+1}})$.


\begin{lemma}
\label{vee-learning}
Assume that $\|\vect{a}^{t}\| \leq G$ for all $1 \leq t \leq T$ and the update scheme is chosen either as the gradient ascent update or 
other procedure with regret of $O(\sqrt{T})$. Then, using the lattice $\mathcal{L}$ with parameter 
$M= (T/n)^{1/4}$ and $\epsilon = 1/\sqrt{T}$ (in the \texttt{round} procedure), 
Algorithm~\ref{algo:vee} achieves a regret bound of 
$O\bigl(G (n T)^{3/4}\bigr)$.
\end{lemma}
\begin{proof}
Observe that  $\vect{y}^{t}$ be the solution of the online linear optimization problem in dimension $(nM)$ at time $t$ where 
the reward function at time $t$ is $\langle \widetilde{\vect{a}}^{t} \circ (\vect{1} - \vect{1}_{C^{t}}), \cdot  \rangle$.  
Moreover, recall that $\vect{x}^{t} \in \mathcal{L} \cap \mathcal{K}$ for every time $t$.
First, by the construction, we have:
\begin{align*}
m(\overline{\vect{c}}^{t} \vee \vect{x}^{t}) &= m(\overline{\vect{c}}^{t}) \vee m(\vect{x}^{t}) = \vect{1}_{C^{t}} \vee \vect{1}_{X^{t}}, \\
\langle \vect{a}^{t}, \overline{\vect{c}}^{t}  \vee \vect{x}^{t} \rangle &= 
 \langle \widetilde{\vect{a}}^{t}, \vect{1}_{C^{t}} \vee \vect{1}_{X^{t}} \rangle.
\end{align*}
Therefore, 
$$
\sum_{t = 1}^{T} \E \bigl[ \langle \vect{a}^{t}, \overline{\vect{c}}^{t}  \vee \vect{x}^{t} \rangle \bigr ]
= \sum_{t = 1}^{T} \E \bigl[  \langle \widetilde{\vect{a}}^{t}, \vect{1}_{C^{t}} \vee \vect{1}_{X^{t}} \rangle  \bigr ]
$$ 
Besides, for every vector $\vect{1}_{C}$ and vector $\vect{z} \in [0,1]^{n \times M}$, the following identity always holds: 
$\vect{1}_{C} \vee \vect{z} = \vect{z} + \vect{1}_{C} \circ (\vect{1} - \vect{z})$.
Therefore, for every vector $\vect{z} \in [0,1]^{n \times M}$,
\begin{align}	\label{eq:vee1}
 \langle \widetilde{\vect{a}}^{t}, \vect{1}_{C^{t}} \vee\vect{y} \rangle
 &=  \langle \widetilde{\vect{a}}^{t}, \vect{z} + \vect{1}_{C^{t}} \circ (\vect{1} - \vect{z}) \rangle
 = \langle \widetilde{\vect{a}}^{t}, \vect{z} \circ (\vect{1} - \vect{1}_{C^{t}}) \rangle
 	+ \langle \widetilde{\vect{a}}^{t},  \vect{1}_{C^{t}} \rangle	\notag \\
&=  \langle \widetilde{\vect{a}}^{t} \circ (\vect{1}- \vect{1}_{C^{t}}), \vect{z}  \rangle
 	+ \langle \widetilde{\vect{a}}^{t},  \vect{1}_{C^{t}} \rangle
\end{align}
where note that the last term $\langle \widetilde{\vect{a}}^{t},  \vect{1}_{C^{t}} \rangle$ is independent of the decision at time $t$. 

By linearity of expectation, we have
\begin{align}	\label{eq:vee2}
\sum_{t = 1}^{T} \E &\bigl[ \langle \vect{a}^{t}, \overline{\vect{c}}^{t}  \vee \vect{x}^{t} \rangle \bigr ]
= \sum_{t = 1}^{T} \E \bigl[  \langle \widetilde{\vect{a}}^{t} \circ (\vect{1} - \vect{1}_{C^{t}}), \vect{1}_{X^{t}}  \rangle
 	+ \langle \widetilde{\vect{a}}^{t},  \vect{1}_{C^{t}} \rangle  \bigr ] 		\notag \\
&= \sum_{t = 1}^{T}  \langle \widetilde{\vect{a}}^{t} \circ (\vect{1} - \vect{1}_{C^{t}}), \vect{y}^{t}  \rangle
	+ \sum_{t = 1}^{T}  \langle \widetilde{\vect{a}}^{t} \circ (\vect{1} - \vect{1}_{C^{t}}), \E[\vect{1}_{X^{t}}] - \vect{y}^{t}  \rangle
 	+ \sum_{t = 1}^{T} \langle \widetilde{\vect{a}}^{t},  \vect{1}_{C^{t}} \rangle 	\notag \\
&\geq \sum_{t = 1}^{T}  \langle \widetilde{\vect{a}}^{t} \circ (\vect{1} - \vect{1}_{C^{t}}), \vect{y}^{t}  \rangle
 	+ \sum_{t = 1}^{T} \langle \widetilde{\vect{a}}^{t},  \vect{1}_{C^{t}} \rangle 
	- \sum_{t = 1}^{T}  \| \widetilde{\vect{a}}^{t} \circ (\vect{1} - \vect{1}_{C^{t}}) \| \cdot \| \E[\vect{1}_{X^{t}}] - \vect{y}^{t}  \|  \notag \\
&\geq \sum_{t = 1}^{T}  \langle \widetilde{\vect{a}}^{t} \circ (\vect{1} - \vect{1}_{C^{t}}), \vect{y}^{t}  \rangle
 	+ \sum_{t = 1}^{T} \langle \widetilde{\vect{a}}^{t},  \vect{1}_{C^{t}} \rangle 
	- T G \epsilon.
\end{align}
The first inequality follows Cauchy-Schwarz inequality. The last inequality is due to the property of \texttt{round} 
and $\| \widetilde{\vect{a}}^{t} \circ (\vect{1} - \vect{1}_{C^{t}}) \| \leq G$ since $\|\vect{a}^{t}\| \leq G$ and the construction 
of $\widetilde{\vect{a}}^{t}$.

Let $\vect{x}^{*}$ be the optimal solution in hindsight for the online vee learning problem. 
Let $\overline{\vect{x}}^{*}$ be the rounded solution of $\vect{x}^{*}$ onto the lattice $\mathcal{L}$. 
Denote $\vect{1}_{X^{*}} = m(\overline{\vect{x}}^{*})$.
By (\ref{eq:vee1}) and (\ref{eq:vee2}) and the choice of $\epsilon = O(1/\sqrt{T})$, we have 
\begin{align*}
\sum_{t = 1}^{T} \E \bigl[ \langle \vect{a}^{t}, \overline{\vect{c}}^{t}  \vee \vect{x}^{t} \rangle \bigr ]
- \sum_{t = 1}^{T} \langle \vect{a}^{t}, \overline{\vect{c}}^{t}  \vee \overline{\vect{x}}^{*} \rangle
\geq \sum_{t = 1}^{T} \E \bigl[ \langle \widetilde{\vect{a}}^{t} \circ (\vect{1} - \vect{1}_{C^{t}}), \vect{y}^{t} - \vect{1}_{X^{*}} \rangle \bigr ]
\geq - O(nMG\sqrt{T})
\end{align*}
where the last inequality is due to the regret bound of the gradient descent or the follow-the-perturbed-leader algorithms
(in dimesion $nM$). As $\|\vect{a}^{t}\| \leq G$, function $\langle \vect{a}^{t}, \cdot \rangle$ is $G$-Lipschitz. Hence,
\begin{align*}
\sum_{t = 1}^{T} &\E \bigl[ \langle \vect{a}^{t}, \vect{c}^{t}  \vee \vect{x}^{t} \rangle \bigr ]
- \sum_{t = 1}^{T} \langle \vect{a}^{t}, \vect{c}^{t}  \vee \vect{x}^{*} \rangle \\
&=
\sum_{t = 1}^{T} \E \bigl[ \langle \vect{a}^{t}, \vect{c}^{t}  \vee \vect{x}^{t} \rangle \bigr ]
- \sum_{t = 1}^{T} \langle \vect{a}^{t}, \overline{\vect{c}}^{t}  \vee \vect{x}^{t} \rangle
+\sum_{t = 1}^{T} \E \bigl[ \langle \vect{a}^{t}, \overline{\vect{c}}^{t}  \vee \vect{x}^{t} \rangle \bigr ]
- \sum_{t = 1}^{T} \langle \vect{a}^{t}, \overline{\vect{c}}^{t}  \vee \overline{\vect{x}}^{*} \rangle \\
	&\qquad +\sum_{t = 1}^{T} \E \bigl[ \langle \vect{a}^{t}, \overline{\vect{c}}^{t}  \vee \overline{\vect{x}}^{*} \rangle \bigr ]
	- \sum_{t = 1}^{T} \langle \vect{a}^{t}, \vect{c}^{t}  \vee \overline{\vect{x}}^{*} \rangle 
	+\sum_{t = 1}^{T} \E \bigl[ \langle \vect{a}^{t}, \vect{c}^{t}  \vee \overline{\vect{x}}^{*} \rangle \bigr ]
	- \sum_{t = 1}^{T} \langle \vect{a}^{t}, \vect{c}^{t}  \vee \vect{x}^{*} \rangle \\
&\geq - O(nMG\sqrt{T}) - O(T G\frac{\sqrt{n}}{M}).
\end{align*}
where 
$\|\vect{c}^{t}  \vee \vect{x}^{t} - \overline{\vect{c}}^{t}  \vee \vect{x}^{t}\| 
	\leq \|\vect{c}^{t}  - \overline{\vect{c}}^{t} \| \leq \sqrt{n}/M$
and similarly $\|\overline{\vect{c}}^{t}  \vee \overline{\vect{x}}^{*} - \vect{c}^{t}  \vee \overline{\vect{x}}^{*}\|$
and $\|\vect{c}^{t}  \vee \overline{\vect{x}}^{*} - \overline{\vect{c}}^{t}  \vee \vect{x}^{*}\|$ 
are bounded by $\sqrt{n}/M$.
Choosing $M = (T/n)^{1/4}$, one gets the regret bound of $O\bigl(G(nT)^{3/4}\bigr)$.
\end{proof}


\subsection{Algorithm for online non-monotone DR-submodular maximization}

In this section, we will provide an algorithm for the problem of online non-monotone DR-submodular maximization.
The algorithm maintains $L$ online vee learning oracles $\mathcal{E}_{1}, \ldots, \mathcal{E}_{L}$. 
At each time step $t$, the algorithm starts from $\vect{0}$ (origin)  and executes $L$ steps of the Frank-Wolfe algorithm  where the update vector $\vect{v}^{t}_{\ell}$ is constructed by combining the output $\vect{u}^{t}_{\ell}$ of the online vee optimization oracle ${\cal E}_\ell$ with the current vector $\vect{x}^{t}_{\ell}$. In particular, $\vect{v}^t_\ell$ is set to be $\vect{x}^t_\ell  \vee \vect{u}^t_\ell - \vect{x}^t_\ell$. 
Then, the solution $\vect{x}^{t}$ is produced at the end of the $L$-th step.
Subsequently, after observing the DR-submodular function $F^{t}$, the algorithm defines a vector $\vect{d}^{t}_{\ell}$ and feedbacks  
$\langle \vect{d}^t_{\ell}, \vect{x}^t_\ell  \vee \vect{u}^t_\ell \rangle$ as the reward function at time $t$ to the oracle $\mathcal{E}_{\ell}$ for $1 \leq \ell \leq L$.
The pseudocode in presented in Algorithm~\ref{algo:online-FW}.

\begin{algorithm}[h]
\begin{flushleft}
\textbf{Input}:  A convex set $\mathcal{K}$, 
	a time horizon $T$, online vee optimization oracles $\mathcal{E}_1, \ldots, \mathcal{E}_L$, 
	step sizes $\rho_\ell \in (0, 1)$, and $\eta_\ell = 1/L$ for all $\ell \in [L]$
\end{flushleft}
\begin{algorithmic}[1]
\STATE Initialize vee optimizing oracle ${\cal E}_\ell$ for all $\ell \in \{1, \ldots L\}$.
\STATE Initialize $\vect{x}^t_{1} \gets \vect{0}$ and $\vect{d}^{t}_{0} \gets \vect{0}$
	for every $1 \leq t \leq T$. 
\FOR {$t = 1$ to $T$}	 		
		\STATE Compute $\vect{u}^{t}_{\ell} \gets$ output of oracle $\mathcal{E}_\ell$ in round $t$ for all $\ell \in \{1, \ldots, L\}$. 
		\FOR{$1 \leq \ell \leq L$}
			\STATE Set $\vect{v}^t_\ell \gets \vect{x}^{t}_{\ell} \vee \vect{u}^t_\ell - \vect{x}^{t}_{\ell}$ 
			\STATE Set $\vect{x}^{t}_{\ell+1} \gets \vect{x}^{t}_{\ell} + \eta_\ell \vect{v}^{t}_{\ell}$.  \label{update-x} 
		\ENDFOR
		\STATE Set $\vect{x}^t \gets \vect{x}^{t}_{L+1}$.
		\STATE Play $\vect{x}^t$, observe the function $F^{t}$ and get the reward of $F^{t}(\vect{x}^t)$. 
		\STATE Compute a gradient estimate $\vect{g}^t_\ell$ such that 
			$\mathbb{E}\big[ \vect{g}^t_\ell |  \vect{x}^t_\ell \big] = \nabla F^t\vect{(x}^t_\ell)$  for all $\ell \in \{1, 2, \ldots, L \}$.	
		\STATE Compute $\vect{d}^t_{\ell} = (1 - \rho_\ell) \vect{d}^t_{\ell-1} + \rho_\ell \vect{g}^t_\ell$, for every $\ell \in \{1, \ldots, L\}$ 
		\STATE Feedback $\langle \vect{d}^t_\ell, \vect{x}^{t}_{\ell} \vee \vect{u}^t_\ell  \rangle$ 
				to oracle $\mathcal{E}_\ell$ 
		for all $\ell \in \{1, 2, \ldots, L\}$, 
\ENDFOR
\end{algorithmic}
\caption{Online algorithm for down-closed convex sets}
\label{algo:online-FW}
\end{algorithm}

We first prove a technical lemma.
\begin{lemma} \label{bound-x-l}
Let $\vect{x}^{t}_{\ell}$ be the vectors constructed in Algorithm~\ref{algo:online-FW}.
Then, it holds that $1 - \| \vect{x}^t_{\ell+1}\|_{\infty} \geq \prod_{\ell' = 1}^\ell (1-\eta_{\ell'})$ for every $t \in \{1, 2, \ldots, T\}$ and for every $\ell \in \{1, 2, \ldots, L\}$.
\end{lemma}

\begin{proof}
We show that inequality holds for every $t$. Fix an arbitrary $t \in \{1, 2, \ldots, T\}$. For the ease of exposition, we drop the index $t$ in the following proof. Let $(x_{\ell})_{i}$, $(u_{\ell})_{i}$ and $(v_{\ell})_{i}$ be the $i^{\text{th}}$ coordinates of 
vectors $\vect{x}_{\ell}$, $\vect{u}_{\ell}$ and $\vect{v}_{\ell}$ for $1 \leq i \leq n$ and for $1 \leq \ell \leq L$, respectively.
We first obtain the following recursion on fixed $(x_\ell)_{i}$. 
\begin{align*}
	1 - (x_{\ell+1})_{i} &= 1 - \left( (x_\ell)_{i} + \eta_\ell (v_\ell)_{i} \right) 
	\geq 1 -  \left( (x_\ell)_{i} + \eta_\ell (1- (x_\ell)_{i}) \right) 
	= (1-\eta_\ell) (1- (x_\ell)_{i} )
\end{align*}
where the inequality holds since $0 \leq (x_{\ell} \vee u_{\ell})_{i} \leq 1$.
Since $(x_{1})_{i} = 0$, we have $1 - (x^t_{\ell+1})_{i} \geq \prod_{\ell' = 1}^\ell (1-\eta_{\ell'})$. 
The lemma holds since the inequality holds for every $1 \leq i \leq n$.
\end{proof}

We are now ready to prove the main theorem of the section.

\begin{theorem}
\label{online-bound}
Let ${\cal K} \subseteq [0,1]^n$ be a down-closed convex set with diameter $D$. 
Assume that functions $F^{t}$'s are $G$-Lipschitz, $\beta$-smooth 
and $\mathbb{E} \left[ \norm{\tilde{\vect{g}}^{t}_{\ell} - \nabla F^{t}(\vect{x}^{t}_{\ell})}^2 \right] \leq \sigma^2$ for 
every $t$ and $\ell$.
Then, the step-sizes $\eta_\ell = 1/L$ and $\rho_\ell = 2/(\ell+3)^{2/3}$ for all $1 \leq \ell \leq L$ ,
Algorithm~\ref{algo:online-FW} achieves the following guarantee:
\begin{align*}
\sum \limits_{t=1}^T \mathbb{E} \big[F^t(\vect{x}^{t}) \big] 
\geq 
\left( \frac{1}{e} - O\left(\frac{1}{L}\right) \right) \sum \limits_{t=1}^T  F^t(\vect{x}^{*}) 
- O(n^{3/4}GT^{3/4}) - O \left(\frac{(\beta D + G + \sigma) DT}{L^{1/3}} \right). 
\end{align*}
Choosing $L = T^{3/4}$ and note that $D \leq \sqrt{n}$, we get
\begin{align}	\label{eq:thm-downclosed}
\sum \limits_{t=1}^T \mathbb{E} \big[F^t(\vect{x}^{t}) \big] 
\geq 
\left( \frac{1}{e} - O\left(\frac{1}{T^{3/4}}\right) \right) \max_{\vect{x} \in \mathcal{K}}\sum _{t=1}^T  F^t(\vect{x}) 
- O\bigl( (G + \sigma) n^{5/4}T^{3/4} \bigr).
\end{align}
\end{theorem}


\begin{proof}
Let $\vect{x}^{*}$ be an optimal solution in hindsight, i.e., 
$\vect{x}^{*} \in \arg \max_{\vect{x} \in \mathcal{K}}\sum _{t=1}^T  F^t(\vect{x})$. 
Fix a time step $t \in \{1, 2, \ldots, T\}$. For every $2 \leq \ell \leq L+1$, we have following: 
\begin{align*}
F^t(\vect{x}^t_{\ell}) &- F^t(\vect{x}^t_{\ell-1}) 
\geq  \langle \nabla F^t(\vect{x}^t_{\ell}), (\vect{x}^t_{\ell} - \vect{x}^t_{\ell-1})  \rangle - (\beta/2) \| \vect{x}^t_{\ell} - \vect{x}^t_{\ell-1} \|^2 \tag{using $\beta$-smoothness} \\ \\
&= \eta_{\ell-1} \langle \nabla F^t(\vect{x}^{t}_{\ell}), \vect{v}^t_{\ell-1} \rangle  - \frac{\beta \eta^2_{\ell-1}}{2} \| \vect{v}^t_{\ell-1} \|^2 \tag{using the update step in the algorithm} \\ \\
&= \eta_{\ell-1} \langle \nabla F^t(\vect{x}^{t}_{\ell-1}), \vect{v}^t_{\ell-1} \rangle  + \eta_{\ell-1} \langle \nabla F^t(\vect{x}^{t}_{\ell}) - \nabla F^t(\vect{x}^{t}_{\ell-1}), \vect{v}^t_{\ell-1} \rangle - \frac{\beta \eta^2_{\ell-1}}{2} \| \vect{v}^t_{\ell-1} \|^2 \\ \\
&\geq \eta_{\ell-1} \langle \nabla F^t(\vect{x}^{t}_{\ell-1}), \vect{v}^t_{\ell-1} \rangle  - \eta_{\ell-1}{\| \nabla F^t(\vect{x}^{t}_{\ell}) - \nabla F(\vect{x}^{t}_{\ell-1}) \|} \cdot {\| \vect{v}^t_{\ell-1} \|} - \frac{\beta \eta^2_{\ell-1}}{2} \| \vect{v}^t_{\ell-1} \| ^2   \tag{Cauchy-Schwarz} \\ \\
&\geq \eta_{\ell-1} \langle \nabla F^t(\vect{x}^{t}_{\ell-1}), \vect{v}^t_{\ell-1} \rangle  - \eta_{\ell-1} \beta {\|\vect{x}^{t}_{\ell} - \vect{x}^{t}_{\ell-1} \|} \cdot {\| \vect{v}^t_{\ell-1} \|} - \frac{\beta \eta^2_{\ell-1}}{2} \| \vect{v}^t_{\ell-1} \| ^2   \tag{using $\beta$-smoothness} \\ \\
&\geq \eta_{\ell-1} \langle \nabla F^t(\vect{x}^{t}_{\ell-1}), \vect{v}^t_{\ell-1} \rangle  - 2\beta \eta^2_{\ell-1} \| {\vect{v}^t_{\ell-1}} \|^2 \\ \\
&\geq \eta_{\ell-1} \langle \nabla F^t(\vect{x}^{t}_{\ell-1}), \vect{v}^t_{\ell-1} \rangle  - 2 \beta \eta^2_{\ell-1} D^2 \tag{diametre of ${\cal K}$ is bounded}\\ \\
&= \eta_{\ell-1} \langle \nabla F^t(\vect{x}^{t}_{\ell-1}), \vect{x}^{t}_{\ell-1} \vee  \vect{x}^{*} - \vect{x}^{t}_{\ell-1} \rangle + \eta_{\ell-1} \langle \nabla F^t(\vect{x}^{t}_{\ell-1}), \vect{v}^{t}_{\ell-1} - (\vect{x}^{t}_{\ell-1} \vee  \vect{x}^{*} - \vect{x}^{t}_{\ell-1}) \rangle \\
 &\qquad  - 2\beta \eta^2_{\ell-1} D^2 \\ \\ 
&= \eta_{\ell-1} \langle \nabla F^t(\vect{x}^{t}_{\ell-1}), \vect{x}^{t}_{\ell-1} \vee  \vect{x}^{*} - \vect{x}^{t}_{\ell-1} \rangle 
+ \eta_{\ell-1} \langle  \vect{d}^t_{\ell-1}, \vect{v}^{t}_{\ell-1} - (\vect{x}^{t}_{\ell-1} \vee  \vect{x}^{*} - \vect{x}^{t}_{\ell-1}) \rangle \\
& \qquad +\eta_{\ell-1} \langle  \nabla F^t(\vect{x}^{t}_{\ell-1}) - \vect{d}^t_{\ell-1}, \vect{v}^{t}_{\ell-1} - (\vect{x}^{t}_{\ell-1} \vee  \vect{x}^{*} - \vect{x}^{t}_{\ell-1}) \rangle 
- 2\beta \eta^2_{\ell-1} D^2 \\ \\ 
&\geq \eta_{\ell-1} (F^t(\vect{x}^{t}_{\ell-1} \vee  \vect{x}^{*})  - F^t(\vect{x}^{t}_{\ell-1} )) 
+ \eta_{\ell-1} \langle  \vect{d}^t_{\ell-1}, \vect{v}^{t}_{\ell-1} - (\vect{x}^{t}_{\ell-1} \vee  \vect{x}^{*} - \vect{x}^{t}_{\ell-1}) \rangle \\
& \qquad +\eta_{\ell-1} \langle  \nabla F^t(\vect{x}^{t}_{\ell-1}) - \vect{d}^t_{\ell-1}, \vect{v}^{t}_{\ell-1} - (\vect{x}^{t}_{\ell-1} \vee  \vect{x}^{*} - \vect{x}^{t}_{\ell-1}) \rangle 
- 2\beta \eta^2_{\ell-1} D^2  \tag{using concavity in positive direction} \\ \\ 
&\geq \eta_{\ell-1} (F^t(\vect{x}^{t}_{\ell-1} \vee  \vect{x}^{*})  - F^t(\vect{x}^{t}_{\ell-1} )) 
+ \eta_{\ell-1} \langle  \vect{d}^t_{\ell-1},\vect{x}^{t}_{\ell-1} \vee \vect{u}^{t}_{\ell-1} - \vect{x}^{t}_{\ell-1} \vee  \vect{x}^{*} \rangle \\
& \qquad +\eta_{\ell-1} \langle  \nabla F^t(\vect{x}^{t}_{\ell-1}) - \vect{d}^t_{\ell-1}, \vect{x}^{t}_{\ell-1} \vee \vect{u}^{t}_{\ell-1} - \vect{x}^{t}_{\ell-1} \vee  \vect{x}^{*} \rangle 
- 2\beta \eta^2_{\ell-1} D^2  \tag{using the definition of $\vect{v}^t_{\ell-1}$}  \\ \\ 
&\geq \eta_{\ell-1} \left( \bigl( 1 - \| \vect{x}^{t}_{\ell-1} \|_{\infty} \bigr) F^t(\vect{x}^{*})  - F^t(\vect{x}^{t}_{\ell-1} ) \right) 
	+ \eta_{\ell-1} \langle  \vect{d}^t_{\ell-1},\vect{x}^{t}_{\ell-1} \vee \vect{u}^{t}_{\ell-1} - \vect{x}^{t}_{\ell-1} \vee  \vect{x}^{*} \rangle \\
& \qquad +\eta_{\ell-1} \langle  \nabla F^t(\vect{x}^{t}_{\ell-1}) - \vect{d}^t_{\ell-1}, \vect{x}^{t}_{\ell-1} \vee \vect{u}^{t}_{\ell-1} - \vect{x}^{t}_{\ell-1} \vee  \vect{x}^{*} \rangle 
- 2\beta \eta^2_{\ell-1} D^2  \tag{using Lemma~\ref{lem:obs2}} \\ \\ 
&\geq \eta_{\ell-1} \prod \limits_{\ell' = 1}^{\ell-2} (1- \eta_{\ell'}) \cdot F^t(\vect{x}^{*})  - \eta_{\ell-1} F^t(\vect{x}^{t}_{\ell-1}) 
	+ \eta_{\ell-1} \langle  \vect{d}^t_{\ell-1},\vect{x}^{t}_{\ell-1} \vee \vect{u}^{t}_{\ell-1} - \vect{x}^{t}_{\ell-1} \vee  \vect{x}^{*} \rangle \\
& \qquad +\eta_{\ell-1} \langle  \nabla F^t(\vect{x}^{t}_{\ell-1}) - \vect{d}^t_{\ell-1}, \vect{x}^{t}_{\ell-1} \vee \vect{u}^{t}_{\ell-1} - \vect{x}^{t}_{\ell-1} \vee  \vect{x}^{*} \rangle 
- 2\beta \eta^2_{\ell-1} D^2  	\tag{using Lemma~\ref{bound-x-l}}  \\ \\
&\geq \eta_{\ell-1} \prod \limits_{\ell' = 1}^{\ell-2} (1- \eta_{\ell'}) \cdot F^t(\vect{x}^{*})  - \eta_{\ell-1} F^t(\vect{x}^{t}_{\ell-1}) 
	+ \eta_{\ell-1} \langle  \vect{d}^t_{\ell-1},\vect{x}^{t}_{\ell-1} \vee \vect{u}^{t}_{\ell-1} - \vect{x}^{t}_{\ell-1} \vee  \vect{x}^{*} \rangle \\
& \qquad - \frac{\eta_{\ell-1}}{2}  \left( \frac{1}{\gamma_{\ell-1}} \| \nabla F^t(\vect{x}^{t}_{\ell-1}) - \vect{d}^t_{\ell-1} \|^2  + \gamma_{\ell-1} \| \vect{x}^{t}_{\ell-1} \vee \vect{u}^{t}_{\ell-1} - \vect{x}^{t}_{\ell-1} \vee  \vect{x}^{*} \|^2 \right)   \\
&\qquad - 2\beta \eta^2_{\ell-1} D^2 
\tag{Young's Inequality, parameter $\gamma_{\ell-1}$ will be defined later} \\ \\
 &\geq \eta_{\ell-1} \prod \limits_{\ell' = 1}^{\ell-2} (1- \eta_{\ell'}) \cdot F^t(\vect{x}^{*})  - \eta_{\ell-1} F^t(\vect{x}^{t}_{\ell-1}) 
	+ \eta_{\ell-1} \langle  \vect{d}^t_{\ell-1},\vect{x}^{t}_{\ell-1} \vee \vect{u}^{t}_{\ell-1} - \vect{x}^{t}_{\ell-1} \vee  \vect{x}^{*} \rangle \\
& \qquad -\frac{\eta_{\ell-1}}{2}  \left( \frac{1}{\gamma_{\ell-1}} \| \nabla F^t(\vect{x}^{t}_{\ell-1}) - \vect{d}^t_{\ell-1} \|^2  + \gamma_{\ell-1} D^2 \right)   
- 2\beta \eta^2_{\ell-1} D^2 
\tag{since  $\| \vect{x}^{t}_{\ell-1} \vee \vect{u}^{t}_{\ell-1} -  \vect{x}^{t}_{\ell-1} \vee  \vect{x}^{*}  \| 
			\leq \| \vect{u}^{t}_{\ell-1} -  \vect{x}^{*}  \| \leq D$.}.
\end{align*} 

From the above we have 
\begin{align}
F^t(\vect{x}^{t}_{\ell}) 
&\geq (1 - \eta_{\ell-1}) F^t(\vect{x}^{t}_{\ell-1}) + \eta_{\ell-1} \prod \limits_{\ell'=1}^{\ell-2} (1- \eta_{\ell'}) \cdot F^t(\vect{x}^{*}) 
	+ \eta_{\ell-1} \langle  \vect{d}^t_{\ell-1},\vect{x}^{t}_{\ell-1} \vee \vect{u}^{t}_{\ell-1} - \vect{x}^{t}_{\ell-1} \vee  \vect{x}^{*} \rangle \notag \\
& \qquad -\frac{\eta_{\ell-1}}{2}  \left( \frac{1}{\gamma_{\ell-1}} \| \nabla F^t(\vect{x}^{t}_{\ell-1}) - \vect{d}^t_{\ell-1} \|^2  + \gamma_{\ell-1} D^2 \right)   
- 2\beta \eta^2_{\ell-1} D^2  \notag \\ \notag \\ 
&= \left(1 - \frac{1}{L} \right) F^t(\vect{x}^{t}_{\ell-1}) + \frac{1}{L} \left(1- \frac{1}{L}\right)^{\ell-2} F^t(\vect{x}^{*}) 
	+ \frac{1}{L} \langle  \vect{d}^t_{\ell-1},\vect{x}^{t}_{\ell-1} \vee \vect{u}^{t}_{\ell-1} - \vect{x}^{t}_{\ell-1} \vee  \vect{x}^{*} \rangle \notag \\
& \qquad - \frac{1}{2L}  \left( \frac{1}{\gamma_{\ell-1}} \| \nabla F^t(\vect{x}^{t}_{\ell-1}) - \vect{d}^t_{\ell-1} \|^2  +  \gamma_{\ell-1} D^2 \right)   
-  \frac{2\beta D^2}{L^2}    \tag{replacing $\eta_\ell = 1/L$}  \\ \nonumber 
\end{align}

Applying recursively on $\ell$, we get:
\begin{align}
\label{important-online-recursion} 
F^t(\vect{x}^{t}_{\ell})
&\geq \frac{\ell-1}{L} \left(1- \frac{1}{L}\right)^{\ell-2}  F^t(\vect{x}^{*}) 
	+ \frac{1}{L} \sum \limits_{\ell'=1}^{\ell-1} \left( \left(1 - \frac{1}{L}\right)^{\ell-1-\ell'} \langle \vect{d}^t_{\ell'}, \vect{x}^{t}_{\ell'} \vee \vect{u}^{t}_{\ell'} - \vect{x}^{t}_{\ell'} \vee  \vect{x}^{*} \rangle  \right)  \nonumber \\
	&\qquad - \frac{1}{2L}  \sum \limits_{\ell'=1}^{\ell-1} \left( \left(1 - \frac{1}{L}\right)^{\ell-1-\ell'} \left( \frac{1}{\gamma_{\ell'}} \| \nabla F^t(\vect{x}^{t}_{\ell'}) - \vect{d}^t_{\ell'} \|^2 +  \gamma_{\ell'} D^2 \right) \right)
	- \frac{2 \beta D^2}{L}  \\ \nonumber 
\end{align}

Summing above inequality (with $\ell = L+1$) over all $t \in [T]$ and using online vee maximization oracle with regret ${\cal R}^T$, we obtain
\begin{align}
&\sum_{t=1}^T F^t(\vect{x}^{t}_{L+1}) \notag \\
&\geq \left(1- \frac{1}{L}\right)^{L-1} \sum_{t=1}^{T} F^t(\vect{x}^{*}) 
	+ \frac{1}{L} \sum \limits_{t=1}^T \sum \limits_{\ell'=1}^{L} \left( \left(1 - \frac{1}{L}\right)^{L-\ell'} 
			\langle \vect{d}^t_{\ell'}, \vect{x}^{t}_{\ell'} \vee \vect{u}^{t}_{\ell'} - \vect{x}^{t}_{\ell'} \vee  \vect{x}^{*} \rangle  \right)
	 \nonumber \\
	&\qquad \qquad - \frac{1}{2L}\sum \limits_{t=1}^T  \sum \limits_{\ell'=1}^L \left( \left(1 - \frac{1}{L}\right)^{L-\ell'} \left( \frac{1}{\gamma_{\ell'}} \| \nabla F^t(\vect{x}^{t}_{\ell'}) - \vect{d}^t_{\ell'} \|^2 +  \gamma_{\ell'} D^2 \right) \right) - \frac{2 \beta T D^2}{L} \nonumber  \\ \nonumber \\
&\geq \left(1- \frac{1}{L}\right)^{L-1}  \sum_{t=1}^{T} F^t(\vect{x}^{*}) 
	+ \frac{1}{L} \sum \limits_{\ell'=1}^{L} \left(1 - \frac{1}{L}\right)^{L-\ell'} 
			\cdot \left( \sum \limits_{t=1}^T \langle \vect{d}^t_{\ell'}, \vect{x}^{t}_{\ell'} \vee \vect{u}^{t}_{\ell'} - \vect{x}^{t}_{\ell'} \vee  \vect{x}^{*} \rangle \right) 
	 \nonumber \\
	&\qquad - \frac{1}{2L}\sum \limits_{t=1}^T  \sum \limits_{\ell'=1}^L \left( \frac{1}{\gamma_{\ell'}} \| \nabla F^t(\vect{x}^{t}_{\ell'}) - \vect{d}^t_{\ell'} \|^2 +  \gamma_{\ell'} D^2 \right) - \frac{2 \beta T D^2}{L}  
							\tag{since $1-  1/L < 1$}  \\ \nonumber \\
&\geq \left(1- \frac{1}{L}\right)^{L-1}  \sum_{t=1}^{T} F^t(\vect{x}^{*}) 
	- \frac{1}{L} \sum \limits_{\ell'=1}^{L}\left(1 - \frac{1}{L}\right)^{L-\ell'} {\cal R}^T 
	- \frac{2 \beta T D^2}{L} \nonumber \\
	&\qquad - \frac{1}{2L}\sum \limits_{t=1}^T  \sum \limits_{\ell'=1}^L  \left( \frac{1}{\gamma_{\ell'}} \| \nabla F^t(\vect{x}^{t}_{\ell'}) - \vect{d}^t_{\ell'} \|^2 +  \gamma_{\ell'} D^2 \right) \tag{${\cal R}^T$ is the regret of a single oracle}  \\ \nonumber \\
&\geq \left( \frac{1}{e} - O\left(\frac{1}{L}\right) \right) \sum_{t=1}^{T} F^t(\vect{x}^{*}) 
	- {\cal R}^T - \frac{2 \beta T D^2}{L}
	- \frac{1}{2L}  \sum \limits_{t=1}^T  \sum \limits_{\ell'=1}^L  \left( \frac{1}{\gamma_{\ell'}} \| \nabla F^t(\vect{x}^{t}_{\ell'}) - \vect{d}^t_{\ell'} \|^2 +  \gamma_{\ell'} D^2 \right)  \label{almost-thm}
\end{align}
\bigskip

\begin{claim}
\label{online-reduce}  
Choosing $\gamma_\ell = \frac{Q^{1/2}}{D (\ell +3)^{1/3}}$ for $1 \leq \ell \leq L$
where $Q  = \{ \max_{1 \leq t \leq T} \| \nabla F^t(\vect{0}) \|^2 4^{2/3}, 4 \sigma^2 + 3(\beta D)^2/2 \}$, it holds that 
\begin{align*}
\sum \limits_{t = 1}^T \sum \limits_{\ell=1}^L   \left( \frac{1}{\gamma_{\ell}} \mathbb{E} \big[ \| \nabla F^t(\vect{x}^{t}_{\ell}) - \vect{d}^t_{\ell} \|^2 \big] +  \gamma_{\ell} D^2 \right)
= O \left((\beta D + G + \sigma) DT L^{2/3} \right) 
\end{align*}
\end{claim}

\begin{claimproof}
Fix a time step $t$. We apply Lemma~\ref{lem:variance-reduction} to the left hand side of the claim inequality where, for $1 \leq \ell \leq L$, $\vect{a}_\ell$ denotes the gradients $\nabla F^t(\vect{x}^t_\ell)$, $\tilde{\vect{a}}_\ell$ denote the stochastic gradient estimate $\vect{g}^t_\ell$. Additionally, $\rho_\ell = \frac{2}{(\ell+3)^{2/3}}$  for $1 \leq \ell \leq L$. 
Hence we get:
\begin{align}
	\mathbb{E} \big[ \| \nabla F^t(\vect{x}^t_\ell) - \vect{d}^t_\ell  \|^2 \big] \leq \frac{Q}{(\ell+4)^{2/3}}.
\end{align}
With the choice $\gamma_\ell = \frac{Q^{1/2}}{D (\ell+4)^{1/3}}$, we obtain:
\begin{align*}
\sum \limits_{\ell=1}^L  & \left( \frac{1}{\gamma_{\ell}}  \mathbb{E} \big[ \| \nabla F^t(\vect{x}^{t}_{\ell}) - \vect{d}^t_{\ell} \|^2 \big] +  \gamma_{\ell} D^2 \right) 
\leq 
2 DQ^{1/2} \sum \limits_{\ell=1}^L  \frac{1}{(\ell + 4)^{1/3}}   \\
&\leq 2 DQ^{1/2} \int \limits_{\ell = 1}^L \frac{1}{\ell^{1/3}} d\ell {~~~} \leq 3 DQ^{1/2} L^{2/3}
= O \left((\beta D + G + \sigma) D L^{2/3} \right)   
\end{align*}
The claim follows by summing the inequality above over all $1 \leq t \leq T$.
\end{claimproof}

Combining Claim~\ref{online-reduce} and Inequality (\ref{almost-thm}) we get:
\begin{align*}
\sum \limits_{t=1}^T \mathbb{E} \big[F^t(\vect{x}^{t}_{L+1}) \big] 
\geq 
\left( \frac{1}{e} - O\left(\frac{1}{L}\right) \right) \sum \limits_{t=1}^T  F^t(\vect{x}^{*}) 
- {\cal R}^T - \frac{2 \beta T D^2}{L} -  O \left(\frac{(\beta D + G + \sigma) DT}{2 L^{1/3}} \right). 
\end{align*}
By Lemma~\ref{vee-learning}, the regret of the online vee oracle ${\cal R}^T = O(G(nT)^{3/4})$. 
Thus, the theorem follows. 
\end{proof}

\begin{remark}	\label{re:down-closed}
\begin{itemize}
\item If a function $F^{t}$ is not smooth then one can consider 
$\hat{F}^{t}_{\delta}$ defined as  
$\hat{F}^{t}_{\delta} (\vect{x}):= \E_{\vect{r} \sim \mathbb{B}}[F^{t}(\vect{x} + \delta \vect{r})]$
where $\mathbb{B}$ is the $n$-dim unit ball and $\vect{r} \sim \mathbb{B}$ denotes an uniformaly random variable 
taken over $\mathbb{B}$. It is known that $\hat{F}^{t}_{\delta}$ is $nG/\delta$-smooth and 
$|\hat{F}^{t}_{\delta}(\vect{x}) - F^{t}_{\delta}(\vect{x})| \leq \delta G$ (for example, see 
\cite[Lemma 2.6]{Hazanothers16:Introduction-to-online}). 
By considering $\hat{F}^{t}_{\delta}$ (instead of $F^{t}$) with $\delta = 1/\sqrt{T}$ and 
choosing $L = T^{9/4}$, one can achieve the same guarantee (\ref{eq:thm-downclosed}) 
stated in Theorem~\ref{online-bound}.
\item In Algorithm~\ref{algo:online-FW}, we assume the knowledge of the time horizon $T$. 
That can be avoided by the standard doubling trick (shown in the appendix) 
where Algorithm~\ref{algo:online-FW} is invoked repeatedly with a doubling time horizon. 
\end{itemize}
\end{remark}


\section{Online DR-Submodular Maximization over Hypercube $[0,1]^{n}$}


\subsection{Reduction to maximizing online discrete submodular functions} 

Building upon the salient ideas of the discretization and lifting the problem into a higher dimension space in Section~\ref{sec:vee}, 
we will transform DR-submodular functions to discrete submodular functions 
in higher space and use an online algorithm in \cite{RoughgardenWang18:An-Optimal-Learning} 
to produce solutions in higher space. Subsequently, we project the latter to the original 
space and prove that, by this unusual continuous-to-discrete approach (in contrast to usual discrete-to-continuous relaxations),  
one can achieve a similar regret guarantee as the one proved in 
\cite{RoughgardenWang18:An-Optimal-Learning} for online discrete submodular maximization. 
Note that, the following discretization is different to that for online vee learning.


\paragraph{Discretization and Lifting.}	
Let $F: [0,1]^{n} \rightarrow [0,1]$ be a non-negative DR-submodular function.  
Let $\mathcal{L}$ be a lattice $\mathcal{L} = \{0, 2^{-M}, 2 \cdot 2^{-M}, \ldots, \ell \cdot 2^{-M}, \ldots, 1\}^{n}$ where $0 \leq \ell \leq 2^{M}$ for some parameter $M$. 
Note that each $x_{i} \in \{0, 2^{-M}, 2 \cdot 2^{-M}, \ldots, \ell \cdot 2^{-M}, \ldots, 1\}$ can be uniquely decomposed as 
$x_{i} = \sum_{j=0}^{M} 2^{-j} y_{ij}$ where $y_{ij} \in \{0,1\}$.
We lift the set $[0,1]^{n} \cap \mathcal{L}$ to the $(n \times (M+1))$-dim space.
Specifically, define a \emph{lifting} map $m': [0,1]^{n} \cap \mathcal{L} \rightarrow \{0,1\}^{n \times (M+1)}$
such that each point $(x_{1}, \ldots, x_{n}) \in \mathcal{K} \cap \mathcal{L}$ is
mapped to the unique point $(y_{10}, \ldots, y_{1M}, \ldots, y_{n0}, \ldots, y_{nM}) \in \{0,1\}^{n \times (M+1)}$
where $x_{i} = \sum_{j=0}^{M} 2^{-j} y_{ij}$.
Define function $\tilde{f}: \{0,1\}^{n \times (M+1)} \rightarrow [0,1]$ such that 
$\tilde{f}(\vect{1}_{S}) := F({m'}^{-1}(\vect{1}_{S}))$; in other words, 
$\tilde{f}(\vect{1}_{S}) = F(\vect{x})$ where $\vect{x} \in [0,1]^{n} \cap \mathcal{L}$ and $\vect{1}_{S} = m'(\vect{x})$. 

\begin{lemma}
\label{sudmod-ext}
If $F$ is DR-submodular then $\tilde{f}$ is a submodular function.
\end{lemma}
\begin{proof}
We will prove that for any vectors $\vect{1}_{S}$ and $\vect{1}_{T}$ such that $S \subset T$
and every element $k$, 
\begin{align}	\label{eq:DR-2-sub}
\tilde{f}(\vect{1}_{T \cup k}) - \tilde{f}(\vect{1}_{T})  
 &\leq \tilde{f}(\vect{1}_{S \cup k}) - \tilde{f}(\vect{1}_{S}) \notag \\
\Leftrightarrow 
F({m'}^{-1}(\vect{1}_{T \cup k})) - F({m'}^{-1}(\vect{1}_{T}))  
 &\leq F ({m'}^{-1}(\vect{1}_{S \cup k})) - F({m'}^{-1}(\vect{1}_{S})).  
\end{align}

The inequality is trivial if $k \in S \subset T$. Assume that $k \notin S \subset T$.
By the definition of the map ${m'}$, as $k \notin S \subset T$, 
${m'}^{-1}(\vect{1}_{S \cup k}) = {m'}^{-1}(\vect{1}_{S}) + {m'}^{-1}(\vect{1}_{k})$, 
${m'}^{-1}(\vect{1}_{T \cup k}) = {m'}^{-1}(\vect{1}_{T}) + {m'}^{-1}(\vect{1}_{k})$,
and ${m'}^{-1}(\vect{1}_{S \cup k}) \leq {m'}^{-1}(\vect{1}_{T \cup k})$. 
Hence, the inequality (\ref{eq:DR-2-sub}) holds by the DR-submodularity of $F$. 
Therefore, $\tilde{f}$ is a submodular function.
\end{proof}

\subsection{Algorithm}

Fix a lattice $\mathcal{L} = \{0, 2^{-M}, 2 \cdot 2^{-M}, \ldots, \ell \cdot 2^{-M}, \ldots, 1\}^{n}$ 
where $0 \leq \ell \leq 2^{M}$ for some parameter $M$ (to be defined later). 
Let \texttt{RW} be the online $(1/2, O(\sqrt{T}))$-regret randomized algorithm \cite{RoughgardenWang18:An-Optimal-Learning}
for online discrete submodular functions on $\{0,1\}^{n \times (M+1)}$.
Initially, set $\vect{x}^{1} \in [0,1]^{n}$ arbitrarily. 
At every time $t \geq 1$, 
\setlist{nolistsep}
\begin{enumerate}[noitemsep]
	\item Play $\vect{x}^{t}$.
	\item Observe function $F^{t}: [0,1]^{n} \rightarrow [0,1]$. 
	Let $\tilde{f}^{t}: \{0,1\}^{n \times (M+1)} \rightarrow [0,1]$ be the corresponding submodular function
	by the construction above. Let $\vect{1}_{S^{t+1}} \in \{0,1\}^{n \times (M+1)}$ be the random solution returned by 
	the algorithm $\texttt{RW}(\tilde{f}^{1}, \ldots, \tilde{f}^{t})$. Set $\vect{x}^{t+1} = {m'}^{-1}(\vect{1}_{S^{t+1}})$.  
\end{enumerate}

\begin{theorem}
Assume that functions $F^{t}$'s are $G$-Lipschitz. Then, by choosing $M = \log T$, 
the above algorithm achieves
\begin{align*}
\sum_{t=1}^{T} F^{t}(\vect{x}^{t}) 
	\geq \frac{1}{2} \max_{\vect{x} \in [0,1]^{n}} \sum_{t=1}^{T} F^{t}(\vect{x}) - O(nG \sqrt{T} \log T ).
\end{align*}
\end{theorem}
\begin{proof}
Let $\vect{x}^{*}$ be the optimal solution in hindsight, i.e., 
$\vect{x}^{*} \in \arg \min_{\vect{x} \in [0,1]^{n}} \sum_{t=1}^{T} F^{t}(\vect{x})$. 
Let $\overline{\vect{x}}^{*}$ be rounded solution of $\vect{x}^{*}$ onto the lattice $\mathcal{L}$, i.e., 
$\overline{x}^{*}_{i}$ is the largest multiple of $1/2^{M}$ which is smaller than $x^{*}_{i}$ for every $1 \leq i \leq n$.
Note that $\|\vect{x}^{*} - \overline{\vect{x}}^{*}\| \leq \frac{\sqrt{n}}{2^{M}}$.
Denote $\vect{1}_{S^{*}} = m(\overline{\vect{x}}^{*})$. 
Algorithm \texttt{RW} \cite[Corollary 3.2]{RoughgardenWang18:An-Optimal-Learning} guarantees that:
\begin{align*}
\sum_{t=1}^{T} \tilde{f}^{t}(\vect{1}_{S^{t}}) 
	\geq \frac{1}{2} \sum_{t=1}^{T} \tilde{f}^{t}(\vect{1}_{S^{*}})  - O(nM\sqrt{T}).
\end{align*}
By the construction of functions $\tilde{f}^{t}$'s and the Lipschitz property of $F^{t}$'s, we get 
\begin{align*}
\sum_{t=1}^{T} F^{t}(\vect{x}^{t})
&= \sum_{t=1}^{T} \tilde{f}^{t}(\vect{1}_{S^{t}}) 
\geq \frac{1}{2} \sum_{t=1}^{T}  \tilde{f}^{t}(\vect{1}_{S^{*}})  - O(nM\sqrt{T}) \\
&= \frac{1}{2} \sum_{t=1}^{T}  F^{t}(\overline{\vect{x}}^{*})  - O(nM\sqrt{T})
\geq \frac{1}{2} \sum_{t=1}^{T}  F^{t}(\vect{x}^{*})  - O\biggl(\frac{\sqrt{n}G}{2^{M}} T \biggr) - O(nM\sqrt{T}).
\end{align*}
Choose $M = \log T$, the theorem follows.
\end{proof}


\section{Maximizing Non-Monotone DR-Submodular Functions over General Convex Sets}		
\label{sec:online-gen}

In this section, we consider the problem of maximizing non-monotone continuous DR-submodular functions over general convex sets. 
We show that beyond down-closed convex sets, the Meta-Frank-Wolfe algorithm, studied by \cite{ChenHassani18:Online-continuous} for monotone DR-submodular functions, with appropriately chosen step sizes provides indeed meaningful guarantees for the problem of online non-monotone DR-submodular maximization over general convex sets. 
Note that no algorithm with performance guarantee was known in the online setting. 


%

\paragraph{Online linear optimization oracles.}
In our algorithms, we use multiple online linear optimization oracles to estimate the gradient of online arriving functions. This idea was originally developed for maximizing monotone DR-submodular functions~\cite{ChenHarshaw18:Projection-Free-Online}. Before presenting algorithms, we recall the online linear optimization problems and corresponding oracles.  In the online linear optimization problem, at every time $1 \leq t \leq T$, the oracle selects $\vect{u}^{t} \in \mathcal{K}$. Subsequently, the adversary reveals a vector $\vect{d}^{t}$ and feedbacks the function $\langle \cdot , \vect{d}^t \rangle$ (and a reward $\langle \vect{u}^t , \vect{d}^t \rangle$) to the oracle. The objective is to minimize the regret. There exists several oracles that guarantee sublinear regret, for example the gradient descent algorithm has the regret of $O(n\sqrt{T})$ (see for example \cite{Hazanothers16:Introduction-to-online}).  

\paragraph{Algorithm description.} At a high level, at every time $t$, our algorithm produces a solution $\vect{x}^t$ by running $L$ steps of the Frank-Wolfe algorithm that uses the outputs of $L$ linear optimization oracles as update vectors. After the algorithm plays $\vect{x}^t$, it observes stochastic gradient estimates at $L$ points in the convex domain. Subsequently, these estimates are averaged with the estimates from the previous round and are fed to the reward functions of $L$ online linear oracles.  The pseudocode is presented in the Algorithm~\ref{algo:online-pf}.

\begin{algorithm}[ht]
\begin{flushleft}
\textbf{Input}:  A convex set $\mathcal{K}$, a time horizon $T$, online linear optimization oracles $\mathcal{E}_1, \ldots, \mathcal{E}_L$, step sizes $\rho_\ell \in (0, 1)$, and $\eta_\ell \in (0,1)$
\end{flushleft}
\begin{algorithmic}[1]
\STATE Initialize online linear optimization oracle ${\cal E}_\ell$ for all $\ell \in \{1, \ldots L\}$.
\STATE Initialize $\vect{x}^t_{1} \gets \vect{x}_{0}$ for some $\vect{x}_{0} \in \mathcal{K}$ and $\vect{d}^{t}_{0} \gets 0$
	for every $1 \leq t \leq T$. Let 
	$\vect{x}_{0} \gets \arg \min_{\vect{x} \in \mathcal{K}} \|\vect{x}\|_{\infty}$.
\FOR{$t = 1$ to $T$}			
		\STATE Set $\vect{v}^{t}_{\ell} \gets$ output of oracle $\mathcal{E}_\ell$ in round $t-1$ for all $\ell \in \{1, \ldots, L\}$.
		\STATE Set $\vect{x}^{t}_{\ell+1} \gets (1 - \eta_\ell) \vect{x}^{t}_{\ell} + \eta_{\ell} \vect{v}^{t}_{\ell}$.  \label{gen-update-x}
		\STATE Play $\vect{x}^{t}=\vect{x}^{t}_{L+1}$.
		\STATE Observe $\vect{g}^t_\ell$ such that $\mathbb{E}\left[ \vect{g}^t_\ell | \vect{x}^t_\ell \right] = \nabla F^t(\vect{x}^t_\ell)$. 
		\STATE Set $\vect{d}^{t}_{\ell}\gets (1-\rho_{\ell}) \cdot \vect{d}^{t}_{\ell-1} + \rho_{\ell} \cdot \vect{g}^t_\ell$ for $\ell =  \{1,\ldots, L\}$.
		\STATE Feedback the reward $\langle \vect{v}^{t}_\ell,\vect{d}^{t}_{\ell} \rangle$ to $\mathcal{E}_{\ell}$ for $\ell =  \{1, \ldots, L\}$.
\ENDFOR
\end{algorithmic}
\caption{Meta-Frank-Wolfe for general convex domains}
\label{algo:online-pf}
\end{algorithm}



We prove first a technical lemma which is in the same line of Lemma~\ref{bound-x-l}.

\begin{lemma}		
\label{lem:bound-x}
Let $(x^{t}_{\ell})_{i}$ be the $i^{\text{th}}$ coordinate of vector $\vect{x}^{t}_{\ell}$ in Algorithm~\ref{algo:online-pf}.
Setting $\kappa = \left(\frac{\ln 3}{2}\right)$ and $\eta_\ell = \left(\frac{\kappa}{\ell H_L}\right), \forall \ell \in \{1, 2, \ldots, L\}$ where $H_L$ is the $L^{th}$ Harmonic number, the following invariant holds true for every $t \in \{1, \ldots, T\}$ and for every coordinate $i \in \{1, \ldots,n\}$
\begin{align*}
1 - (x^{t}_{L+1})_{i} \geq e^{-\kappa(1 + O(1/\ln^{2}L))} \cdot (1 - (x^{t}_{1})_{i})
\end{align*}
\end{lemma}
\begin{proof}
Fix an arbitrary time step $t \in \{1, \cdots, T\}$. 
For the ease of exposition, we drop the index $t$ in the following proof.  
Using the update step \ref{gen-update-x} from Algorithm~\ref{algo:online-pf}, for every $1 \leq \ell \leq L$ and for every
coordinate $1 \leq i \leq n$, we have  
\begin{align*}
1 - (x_{\ell+1})_{i} &= 1 - (1 - \eta_\ell) (x_{\ell})_{i} - \eta_{\ell} (v_{\ell})_{i} \geq  1 - (1 - \eta_\ell)  (x_{\ell})_{i} - \eta_{\ell} = (1- \eta_{\ell}) (1 - (x_{\ell})_{i}) \\
&\geq e^{-\eta_{\ell} - \eta_{\ell}^{2}} \cdot (1 - (x_{\ell})_{i})
\end{align*}
where we use inequalities $(v_{\ell})_{i} \leq 1$ and $1 - u \geq e^{-u - u^{2}}$ for $0 \leq u < 1/2$; and $\eta_\ell = \frac{\kappa}{\ell H_L}$. Therefore, we get
\begin{align*}
1 - (x_{L+1})_{i} &\geq e^{-\sum_{\ell'=1}^{L} \eta_{\ell'} -\sum_{\ell'=1}^{L} \eta_{\ell'}^{2}} \cdot (1 - (x_1)_{i}) 
	\geq e^{-\kappa(1 + O(1/\ln^{2}L))} \cdot (1 - (x_1)_{i})
\end{align*}
where the last inequality is due to the fact that 
$\sum_{\ell'=1}^{L} \eta_{\ell'} = \kappa$ 
and
$\sum_{\ell'=1}^{L} \eta_{\ell'}^{2} = \sum_{\ell'=1}^{\ell} \frac{\kappa^{2}}{{\ell'}^{2} H_{L}^{2}} = O(1/\ln^{2}L)$. 
\end{proof}
\medskip

%
%
%
%

%

\begin{theorem}	\label{thm:online-general}
Let $\mathcal{K} \subseteq [0,1]^n$ be a convex set with the diameter $D$. Assume that for every $1 \leq t \leq T$, 
\setlist{nolistsep}
\begin{enumerate}[noitemsep]
\item $F^{t}$ is $\beta$-smooth DR-submodular function and the norm of the gradient $\nabla F^{t}$ 
	is bounded by $G$, i.e., $\norm{\nabla F(\vect{x}) } \leq G, \forall \vect{x} \in {\cal K}$, 
\item the variance of the unbiased stochastic gradients is bounded by $\sigma^2$, 
	i.e., $\mathbb{E} \left[ \norm{\vect{g}^{t}_{\ell} - \nabla F^{t}(\vect{x}^{t}_{\ell})}^2\right] \leq \sigma^2$ for 
	every  $1 \leq \ell \leq L$; and 
\item the online linear optimization oracles used in Algorithm~\ref{algo:online-pf} have regret at most $\mathcal{R}^{\mathcal{E}}_T$.
\end{enumerate} 
Then for every $1 \leq \ell \leq L$ setting $\kappa = \frac{\ln 3}{2}$, $\rho_\ell = \frac{2}{{(\ell + 3)}^{2/3}}$ 
and $\eta_\ell = \frac{\kappa}{\ell H_L}$ where $H_L$ is the $L^{th}$ harmonic number, 
the following inequality holds true for Algorithm \ref{algo:online-pf}:
\begin{align*}
\sum \limits_{t = 1}^{T} \mathbb{E} \left[F^t(\vect{x}^t)\right] 
\geq & \left(\frac{1}{3\sqrt{3}} \right) (1 -\| \vect{x}_{0} \|_{\infty}) \max_{\vect{x}^* \in {\cal K}} \sum \limits_{t = 1}^{T} F^t(\vect{x}^*) - O \biggl ( \frac{(\beta D + G + \sigma) D T}{\ln L} \biggr)- O(\mathcal{R}^{\mathcal{E}}_T). 
\end{align*} 
In particular, if $\vect{0} \in \mathcal{K}$ (for e.g. conic convex sets)  then choosing $\vect{x}_0 = \vect{0}$, $L = O(T)$ and online linear algorithm with guarantee 
regret guarantee $R^{\mathcal{E}}_T = O(\sqrt{T})$ (for eg., Online Gradient Descent), 
Algorithm \ref{algo:online-pf} achieves
\begin{align*}
\sum \nolimits_{t = 1}^{T} \mathbb{E} \left[F^t(\vect{x}^t)\right] 
\geq & \frac{1}{3\sqrt{3}}  \max_{\vect{x}^* \in {\cal K}}  \sum \nolimits_{t = 1}^{T} F^t(\vect{x}^*) - O \biggl ( \frac{(\beta D + G + \sigma) D T}{\ln T} \biggr) - O(\sqrt{T}). 
\end{align*} 
\end{theorem}
\begin{proof}
Let $\vect{x}^{*} \in \mathcal{K}$ be a solution that maximizes $\sum \limits_{t = 1}^T F^t(\vect{x})$.
Let 
$$
r =  e^{-\kappa(1 + O(1/\ln^{2}L))} \cdot (1 - \max_{i} (x^{t}_{1})_{i}) = e^{-\kappa(1 + O(1/\ln^{2}L))} \cdot (1 - \max_{i} (x_{0})_{i}).
$$ 
Note that from Lemma~\ref{lem:bound-x}, we have that $(1 - \|\vect{x}^{t}_\ell\|_{\infty}) \geq r$ for every $1 \leq t \leq T$ 
and every $1 \leq \ell \leq L+1$.

We prove the following claim  which is crucial in our analysis.

\begin{claim} 
\label{c-1-1}
For every $1 \leq t \leq T$, it holds that 
\begin{align*}
2F^t&(\vect{x}^t_{L+1}) - r F^t( \vect{x}^{*} ) \\
&\geq \biggl(\prod \limits_{\ell = 1}^{L} (1-2\eta_{\ell}) \biggr) \left( 2F^t(\vect{x}^t_{1})  - r F^t( \vect{x}^{*} ) \right) 
		- \sum \limits_{\ell = 1}^{L} 	 \eta_{\ell} \biggl( \frac{1}{\gamma_{\ell}} \norm{ \nabla F^t\bigl(\vect{x}^t_{\ell}\bigr)-\vect{d}^t_\ell}^2 + \gamma_{\ell} D^2 \biggr) \\ 	 
		& + \sum \limits_{\ell = 1}^{L} 2\eta_{\ell} \biggl(\prod \limits_{\ell' = \ell+1}^{L} ( 1- 2 \eta_{\ell'})\biggr) \langle \vect{d}^t_\ell, ( \vect{x}^{*} - \vect{v}^t_{\ell}) \rangle - 3\beta D^2 \sum \limits_{\ell = 1}^{L}\eta_{\ell}^{2}
\end{align*}
where $D$ is the diameter of $\mathcal{K}$ and $\gamma_\ell$ is any constant greater than $0$. 
\end{claim}

\begin{claimproof}
Fix a time step $t \in \{1, \cdots, T\}$. For the ease of exposition, we drop the time index $t$ in equations. 
For every $1 \leq \ell \leq L$, we have
\begin{align*}
&2F (\vect{x}_{\ell+1}) - r F( \vect{x}^{*} ) \\
&= 2F \bigl((1- \eta_{\ell}) \vect{x}_{\ell} + \eta_{\ell} \vect{v}_{\ell}\bigr) - r F( \vect{x}^{*} ) \tag{the step~\ref{gen-update-x} of Algorithm~\ref{algo:online-pf}}& &  \\
&\geq 2\left(F (\vect{x}_{\ell}) - \eta_{\ell} \bigl \langle \nabla F \bigl((1 - \eta_{\ell}) \vect{x}_{\ell} + \eta_{\ell} \vect{v}_{\ell}\bigr),  \vect{x}_{\ell} - \vect{v}_{\ell} \bigr \rangle - \frac{\beta}{2} (\eta_{\ell})^{2} \| \vect{x}_{\ell} - \vect{v}_{\ell} \|^{2}\right) - r F( \vect{x}^{*} )  
\tag{using $\beta$-smoothness} \\
\\
&= 2F (\vect{x}_{\ell}) - 2\eta_{\ell} \langle \nabla F \bigl((1 - \eta_{\ell}) \vect{x}_{\ell} + \eta_{\ell} \vect{v}_{\ell}\bigr), \vect{x}_{\ell} - \vect{v}_{\ell} \rangle - \beta (\eta_{\ell})^{2} \| \vect{x}_{\ell} - \vect{v}_{\ell} \|^{2} - r F\bigl( \vect{x}^{*} \bigr)  \\
\\
&= \left(2F (\vect{x}_{\ell})- r F( \vect{x}^{*}) \right)   
		- 2\eta_{\ell} \bigl \langle \nabla F\bigl((1 - \eta_{\ell}) \vect{x}_{\ell} + \eta_{\ell} \vect{v}_{\ell}\bigr) 
				-  \nabla F(\vect{x}_{\ell}), \vect{x}_{\ell} - \vect{v}_{\ell} \bigr \rangle \\
				& \qquad  \qquad - 2\eta_{\ell} \bigl\langle \nabla F (\vect{x}_{\ell}), \vect{x}_{\ell} - \vect{v}_{\ell} \bigr\rangle
		  - \beta (\eta_{\ell})^{2} \| \vect{x}_{\ell} - \vect{v}_{\ell}\|^{2}
		\\
\\
&\geq \left( 2F(\vect{x}_{\ell}) - r F(\vect{x}^{*} ) \right) 
		- 2\eta_{\ell} \norm{\nabla F\bigl((1 - \eta_{\ell}) \vect{x}_{\ell} + \eta_{\ell} \vect{v}_{\ell}\bigr) 
				-  \nabla F (\vect{x}_{\ell})} \norm{\vect{x}_{\ell} - \vect{v}_{\ell}} \\
		& \qquad  \qquad  - 2\eta_{\ell} \bigl \langle \nabla F(\vect{x}_{\ell}), \vect{x}_{\ell} - \vect{v}_{\ell} \bigr \rangle - \beta (\eta_{\ell})^{2} \| \vect{x}_{\ell} - \vect{v}_{\ell}\|^{2} \tag{Applying Cauchy-Schwarz} \\
\\
&\geq \left( 2F(\vect{x}_{\ell}) - r F(\vect{x}^{*} ) \right) 
		- 2\eta_{\ell} \beta \norm{(1 - \eta_{\ell}) \vect{x}_{\ell} + \eta_{\ell} \vect{v}_{\ell}) 
				-  \vect{x}_{\ell}} \norm{(\vect{x}_{\ell} - \vect{v}_{\ell})} \\
		& \qquad  \qquad  - 2\eta_{\ell} \bigl \langle \nabla F(\vect{x}_{\ell}), (\vect{x}_{\ell} - \vect{v}_{\ell}) \bigr \rangle - \beta (\eta_{\ell})^{2} \| \vect{x}_{\ell} - \vect{v}_{\ell}\|^{2} \tag{$\beta$-smoothness as gradient Lipschitz} \\
\\
&\geq \left( 2F(\vect{x}_{\ell}) - r F(\vect{x}^{*} ) \right) - 2\eta_{\ell} \bigl \langle \nabla F(\vect{x}_{\ell}), (\vect{x}_{\ell} - \vect{v}_{\ell}) \bigr \rangle - 3 \beta (\eta_{\ell})^{2} \| \vect{x}_{\ell} - \vect{v}_{\ell}\|^{2}  \\
&= \left( 2F(\vect{x}_{\ell}) - r F(\vect{x}^{*} ) \right) + 2\eta_{\ell} \bigl \langle \nabla F(\vect{x}_{\ell}), (\vect{v}_{\ell} - \vect{x}_{\ell}) \bigr\rangle - 3 \beta (\eta_{\ell})^{2} \norm{ \vect{x}_{\ell} - \vect{v}_{\ell} }^{2} \\
\\
&= \left( 2F(\vect{x}_{\ell}) - r F( \vect{x}^{*}) \right) 
		+  2\eta_{\ell} \bigl\langle \nabla F(\vect{x}_{\ell}), (\vect{x}^{*} - \vect{x}_{\ell}) \bigr \rangle + 2\eta_{\ell} \langle \nabla F(\vect{x}_{\ell}), (\vect{v}_{\ell} - \vect{x}^{*}) \rangle - 3\beta (\eta_{\ell})^{2} \norm{ \vect{v}_{\ell} - \vect{x}_{\ell} }^{2} \\
\\
&= \left( 2F(\vect{x}_{\ell}) - r F( \vect{x}^{*}) \right) 
		+  2\eta_{\ell}\bigl  \langle \nabla F(\vect{x}_{\ell}), (\vect{x}^{*} - \vect{x}_{\ell}) \bigr \rangle -  2\eta_{\ell} \bigl \langle \nabla F(\vect{x}_{\ell})-\vect{d}_\ell, (\vect{x}^{*} - \vect{v}_{\ell}) \bigr \rangle & & \\
		& \qquad \qquad - 2\eta_{\ell} \langle \vect{d}_\ell, (\vect{x}^{*} - \vect{v}_{\ell}) \rangle - 3\beta (\eta_{\ell})^{2} \norm{ \vect{v}_{\ell} - \vect{x}_{\ell} }^{2} \\
\\
&\geq \left(2F(\vect{x}_{\ell}) - r F( \vect{x}^{*} ) \right)
		+  2\eta_{\ell} \left( F(\vect{x}^{*} \vee \vect{x}_{\ell}) - 2F(\vect{x}_{\ell}) \right) - 2\eta_{\ell} \bigl \langle \nabla F(\vect{x}_{\ell})-\vect{d}_\ell, ( \vect{x}^{*} - \vect{v}_{\ell}) \bigr \rangle & & \\
		& \qquad \qquad - 2\eta_{\ell} \bigl \langle \vect{d}_\ell, ( \vect{x}^{*} - \vect{v}_{\ell}) \bigr \rangle - 3\beta (\eta_{\ell})^{2} \norm{ \vect{v}_{\ell} - \vect{x}_{\ell} }^{2} 
		 \tag{Lemma \ref{lem:prop2}}  \\
\\
&\geq \left(2F(\vect{x}_{\ell}) - r F( \vect{x}^{*} ) \right)
		+  2\eta_{\ell} \left( (1 - \| \vect{x}_{\ell}\|_{\infty}) F(\vect{x}^{*}) - 2F(\vect{x}_{\ell}) \right) -  2\eta_{\ell} \bigl \langle \nabla F(\vect{x}_{\ell})-\vect{d}_\ell, (\vect{x}^{*} - \vect{v}_{\ell} ) \bigr \rangle & & \\
		& \qquad \qquad - 2\eta_{\ell} \bigl \langle \vect{d}_\ell, ( \vect{x}^{*} - \vect{v}_{\ell}) \bigr \rangle - 3\beta (\eta_{\ell})^{2} \norm{ \vect{v}_{\ell} - \vect{x}_{\ell} }^{2} 
		  \tag{Lemma \ref{lem:obs2}}  \\
\\
&\geq \left(2F(\vect{x}_{\ell}) - r F( \vect{x}^{*} ) \right)
		+  2\eta_{\ell} \left( r F(\vect{x}^{*}) - 2F(\vect{x}_{\ell}) \right) -  2\eta_{\ell} \bigl \langle \nabla F(\vect{x}_{\ell})-\vect{d}_\ell, (\vect{x}^{*} - \vect{v}_{\ell} ) \bigr \rangle & & \\
		& \qquad \qquad - 2\eta_{\ell} \bigl \langle \vect{d}_\ell, ( \vect{x}^{*} - \vect{v}_{\ell}) \bigr \rangle - 3\beta (\eta_{\ell})^{2} \norm{ \vect{v}_{\ell} - \vect{x}_{\ell} }^{2} 
		 \tag{Lemma \ref{lem:bound-x}} \\
\\
&= (1-2\eta_{\ell}) \left(2F(\vect{x}_{\ell})  - r F( \vect{x}^{*} ) \right)
			 -  2\eta_{\ell} \bigl \langle \nabla F\bigl(\vect{x}_{\ell}\bigr)-\vect{d}_\ell, ( \vect{x}^{*} - \vect{v}_{\ell} ) \bigr \rangle & & \\
		& \qquad \qquad - 2\eta_{\ell} \bigl \langle \vect{d}_\ell, (\vect{x}^{*} - \vect{v}_{\ell} ) \bigr \rangle - 3\beta (\eta_{\ell})^{2} \norm{ \vect{v}_{\ell} - \vect{x}_{\ell}}^{2}    \\
\\
%
%
&\geq (1-2\eta_{\ell}) \left( 2F(\vect{x}_{\ell})  - r F( \vect{x}^{*} ) \right)
		 -  \eta_{\ell} \left( \frac{1}{\gamma_{\ell}} \norm{ \nabla F\bigl(\vect{x}_{\ell}\bigr)-\vect{d}_\ell}^2 + \gamma_{\ell} \norm{ \vect{v}_{\ell} - \vect{x}^{*} }^2 \right) & \\
		& \qquad \qquad - 2\eta_{\ell} \langle \vect{d}_\ell, (\vect{x}^{*} - \vect{v}_{\ell}) \rangle - 3\beta (\eta_{\ell})^{2} \norm{ \vect{v}_{\ell} - \vect{x}_{\ell} }^{2} 
		 \tag{using Young's Inequality}  
\end{align*}
where in the last inequality, $\eta_{\ell}$'s are parameters to be defined later (specifically in Claim~\ref{c-1-3}). 
Applying the above inequality for $1 \leq \ell \leq L$ and note that the diameter of $\mathcal{K}$ is bounded by $D$, 
we get
\begin{align*}
2F(\vect{x}_{L+1}) - r F( \vect{x}^{*} )
&\geq 
\prod \limits_{\ell = 1}^{L} (1-2\eta_{\ell})  \left( 2F(\vect{x}_{1})  - r F( \vect{x}^{*}) \right) \\
		& \qquad \qquad - \sum \limits_{\ell = 1}^{L} \biggl(\prod \limits_{\ell' = \ell+1}^{L} ( 1- 2 \eta_{\ell'})\biggr) \eta_{\ell} \left( \frac{1}{\gamma_{\ell}} \norm{ \nabla F(\vect{x}_{\ell})-\vect{d}_\ell}^2 + \gamma_{\ell} D^2 \right)  \\ 	 
		&\qquad \qquad 
		- \sum \limits_{\ell = 1}^{L} 2\eta_{\ell} \biggl(\prod \limits_{\ell' = \ell+1}^{L} ( 1- 2 \eta_{\ell'})\biggr) 
					 \langle \vect{d}_\ell, (\vect{x}^{*} - \vect{v}_{\ell}) \rangle 
		- 3\beta D^2  \sum \limits_{\ell = 1}^{L}\eta^2_{\ell} \\
&\geq 
\prod \limits_{\ell = 1}^{L} (1-2\eta_{\ell})  \left( 2F(\vect{x}_{1})  - r F( \vect{x}^{*}) \right) \\
		& \qquad \qquad - \sum \limits_{\ell = 1}^{L} 
		\eta_{\ell} \left( \frac{1}{\gamma_{\ell}} \norm{ \nabla F(\vect{x}_{\ell})-\vect{d}_\ell}^2 + \gamma_{\ell} D^2 \right)  \\ 	 
		&\qquad \qquad 
		- \sum \limits_{\ell = 1}^{L} 2\eta_{\ell} \biggl(\prod \limits_{\ell' = \ell+1}^{L} ( 1- 2 \eta_{\ell'})\biggr) 
					 \langle \vect{d}_\ell, (\vect{x}^{*} - \vect{v}_{\ell} ) \rangle 
		- 3\beta D^2  \sum \limits_{\ell = 1}^{L}\eta^2_{\ell}
\end{align*}
where the second inequality holds since $ \biggl(\prod \limits_{\ell' = \ell + 1}^{L} ( 1- 2 \eta_{\ell'})\biggr) \leq 1$.
Hence, the claim follows 
\end{claimproof} 

Summing the inequality in Claim~\ref{c-1-1} over all $t = \{1, 2, \ldots, T\}$, we get
\begin{align}	\label{break-eq}
&\sum \limits_{t =1}^{T} \left(2F^t (\vect{x}^t_{L+1}) - r F^t( \vect{x}^{*} ) \right) \nonumber \\
&\qquad \geq \sum \limits_{t = 1}^{T} \prod \limits_{\ell = 1}^{L} (1-2\eta_{\ell})  \left( 2F^t(\vect{x}^t_{1})  - r F^t( \vect{x}^{*}) \right) 
		- \sum \limits_{t = 1}^{T} \sum \limits_{\ell = 1}^{L} \eta_{\ell} \left( \frac{1}{\gamma_{\ell}} \norm{ \nabla F^t(\vect{x}^t_{\ell})-\vect{d}^t_\ell}^2 + \gamma_{\ell} D^2 \right) \nonumber \\ 	 
		&\qquad \qquad 
		- \sum \limits_{\ell = 1}^{L} 2\eta_{\ell} \biggl(\prod \limits_{\ell' = \ell + 1}^{L} ( 1- 2 \eta_{\ell'})\biggr) 
					\sum \limits_{t=1}^{T} \langle \vect{d}^t_\ell, (\vect{x}^{*} - \vect{v}^t_{\ell}) \rangle 
		- 3\beta D^2 T \sum \limits_{\ell = 1}^{L}\eta^2_{\ell} 
\end{align}
Next, we bound the terms on the right hand side of Inequation (\ref{break-eq}) separately
by the following claims. 

\bigskip

\begin{claim} 	\label{c-1-2}
It holds that 
\begin{align*}
	\sum \limits_{t = 1}^{T} \prod \limits_{\ell = 1}^{L} (1-2\eta_{\ell})  \left( 2F^t(\vect{x}^t_{1})  - r F^t( \vect{x}^{*}) \right)   \geq \sum \limits_{t = 1}^{T} e^{-2 \kappa (1 + O(1/\ln^2 L))} \left( 2F^t(\vect{x}^t_{1})  - r F^t( \vect{x}^{*}) \right).
\end{align*}
\end{claim}
\begin{claimproof}
Using the inequality $1 - u \geq e^{-(u + u^2)}$, we have that 
\begin{align*}
	&\sum \limits_{t = 1}^{T} \prod \limits_{\ell = 1}^{L} (1-2\eta_{\ell})  \left( 2F^t(\vect{x}^t_{1})  - r F^t( \vect{x}^{*}) \right) \geq \sum \limits_{t = 1}^{T} e^{-2 \sum \limits_{\ell = 1}^{L} \eta_\ell - 4 \sum \limits_{\ell = 1}^{L} \eta^2_\ell} \left( 2F^t(\vect{x}^t_{1})  - r F^t( \vect{x}^{*}) \right) \\
	&\qquad \qquad  \geq \sum \limits_{t = 1}^{T} e^{-2 \kappa (1 + O(1/\ln^2 L))} \left( 2F^t(\vect{x}^t_{1})  - r F^t( \vect{x}^{*}) \right)
\end{align*}
where the last inequality is due to the facts that 
$\sum_{\ell=1}^{L} \eta_{\ell} = \sum_{\ell=1}^{L} \frac{\kappa}{\ell H_{L}} = \kappa$ and
$\sum_{\ell=1}^{L} \eta_{\ell}^{2} = \sum_{\ell=1}^{L} \frac{\kappa^{2}}{\ell^{2} H_{L}^{2}} = O(1/\ln^{2}L)$
for a constant $\kappa$.
\end{claimproof}
\bigskip 

\begin{claim} \label{c-1-3}
Choose $\gamma_\ell = \frac{Q^{1/2}}{D(\ell + 3)^{1/3}}$ where 
$Q = \max \{ \max_{1 \leq t \leq T} \norm{ \nabla F^t(\vect{x}_{0})}^2 4^{2/3}, 4 \sigma^2 + (3/2) \beta^2 D^2 \}
= O(G^{2} + \sigma^{2} + \beta^{2}D^{2})$
for $1 \leq \ell \leq L-1$, it holds that 
\begin{align*}
	\sum \limits_{t = 1}^{T} \sum \limits_{\ell = 1}^{L} \eta_{\ell} \left( \frac{1}{\gamma_{\ell}} \mathbb{E} \left[ \norm{ \nabla F^t(\vect{x}^t_{\ell})-\vect{d}^t_\ell}^2 \right] + \gamma_{\ell} D^2 \right) \leq \frac{3TDQ^{1/2}}{H_L}.
\end{align*}
\end{claim}
\begin{claimproof}
Fix a time step $t$.
We apply Lemma~\ref{lem:variance-reduction} to left hand side of the inequality where, for $1 \leq \ell \leq L$, $\vect{a}_{\ell}$ 
denotes the gradients $\nabla F(\vect{x}^t_{\ell})$, $\tilde{\vect{a}_{\ell}}$ denote the stochastic gradient 
estimate $\vect{g}^t_{\ell}$, and $\vect{d}_{\ell}$ denotes the vector $\vect{d}^{t}_{\ell}$ in the algorithm.
Additionally, $\rho_\ell = \left(\frac{2}{{(\ell + 3)}^{2/3}}\right)$ for $1 \leq \ell \leq L$. 

We verify the conditions of Lemma~\ref{lem:variance-reduction}. By the algorithm,
\begin{align*}
\norm{\nabla F(\vect{x}^t_{\ell}) - \nabla F(\vect{x}^t_{\ell-1})}
&= \norm{\nabla F\bigl( (1 - \eta_\ell) \vect{x}^{t}_{\ell} + \eta_{\ell} \vect{v}^{t}_{\ell} \bigr) - \nabla F(\vect{x}^t_{\ell-1})} \\
&\leq \beta \eta_{\ell} \| \vect{v}^{t}_{\ell} - \vect{x}^{t}_{\ell} \|
\leq \beta D \left( \frac{\kappa}{\ell H_{L}} \right) \leq \beta D \left(\frac{\kappa}{\ell + 3}\right). 
\end{align*}
Moreover, by the theorem assumption, 
$\mathbb{E} \left[ \norm{\vect{g}^{t}_{\ell} - \nabla F(\vect{x}^{t}_{\ell})}^2\right] \leq \sigma^2$.
Hence, applying Lemma~\ref{lem:variance-reduction}, we get
\begin{align}
	\mathbb{E} \left[ \norm{ \nabla F^t(\vect{x}^t_{\ell})-\vect{d}^t_\ell}^2 \right]  \leq \frac{Q}{{(\ell + 4)}^{2/3}}.  \label{var}
\end{align}

Summing Equation (\ref{var}) for $1 \leq t \leq T$ and $1 \leq \ell \leq L$ and 
for $\gamma_\ell = \frac{Q^{1/2}}{D(\ell + 3)^{1/3}}$ and $\eta_\ell = \frac{\kappa}{\ell H_L}$, we obtain 

\begin{align*}
\sum \limits_{t = 1}^{T} \sum \limits_{\ell = 1}^{L} \eta_{\ell} & \left( \frac{1}{\gamma_{\ell}} 
	\mathbb{E} \left[ \norm{ \nabla F^t(\vect{x}^t_{\ell})-\vect{d}^t_\ell}^2 \right] + \gamma_{\ell} D^2 \right) 
	\leq \frac{T D Q^{1/2}}{H_L} \cdot \sum \limits_{\ell = 1}^L \frac{1}{{\ell(\ell + 4)}^{1/3}} \\
&\leq \frac{T D Q^{1/2}}{H_L} \cdot \int \limits_{1}^{L+1} \frac{1}{{\ell}^{4/3}} d\ell 
\leq \frac{T D Q^{1/2}}{H_L} \left( 3 - \frac{3}{(L+1)^{1/3}}\right)
= O \biggl ( \frac{(\beta D + G + \sigma) D T}{\log L} \biggr)
\end{align*}
\end{claimproof}

\begin{claim}	\label{c-1-4}
It holds that
$$
\sum \limits_{\ell = 1}^{L} 2\eta_{\ell} \biggl(\prod \limits_{\ell' = \ell + 1}^{L} ( 1- 2 \eta_{\ell'})\biggr) 
			\sum \limits_{t=1}^{T} \langle \vect{d}^t_\ell, (\vect{x}^{*} - \vect{v}^t_{\ell} ) \rangle 
			\leq 2 \kappa \mathcal{R}^{\mathcal{E}}_T.
$$
\end{claim}
\begin{claimproof}
The claim follows from  the definition of the regret for the online linear optimization oracles that is, $\sum \limits_{t = 1}^T \langle \vect{d}^t_\ell, (\vect{x}^* - \vect{v}^t_\ell) \rangle \leq \mathcal{R}^{\mathcal{E}}_T$,
and  $\sum \limits_{\ell = 1}^{L} \eta_\ell \leq  \kappa$ and $\prod \limits_{\ell' = \ell}^{L} ( 1- 2 \eta_{\ell'}) \leq 1$.
\end{claimproof}

Taking the expectation on Inequality (\ref{break-eq}) and
using Claims~\ref{c-1-2}, \ref{c-1-3}, \ref{c-1-4} , we have
\begin{align*}
\sum \limits_{t =1}^{T} \left(2 \mathbb{E} \left[F^t(\vect{x}^t_{L+1})\right] - r F^t( \vect{x}^{*} ) \right)
&\geq 
\sum \limits_{t = 1}^{T} e^{-2 \kappa (1 + O(1/\ln^2 L))} \left( 2 F^t(\vect{x}^t_1) - r F^t( \vect{x}^{*}) \right) \\
& \qquad \qquad - O \biggl ( \frac{(\beta D + G + \sigma) D T}{\ln L} \biggr)
- 2 \kappa \mathcal{R}^{\mathcal{E}}_T
- 3\beta D^2 T \sum \limits_{\ell = 1}^{L-1}\eta^2_{\ell} 
\end{align*} 

Rearranging terms and note that $\sum_{\ell=1}^{L-1} \eta_{\ell}^{2} = O(1/\ln^{2}L)$ and $F^t(\vect{x}^t_1) \geq 0$, we get
\begin{align*}
\sum \limits_{t = 1}^{T} \mathbb{E} \left[F^t(\vect{x}^t_{L+1})\right] 
\geq & \frac{r ( 1 - e^{-2 \kappa (1 + O(1/\ln^2 L)} )}{2} \sum \limits_{t = 1}^{T} F^t(\vect{x}^*) \\
	& \qquad \qquad	
	-O \biggl ( \frac{(\beta D + G + \sigma) D T}{\ln L} \biggr) - 2 \kappa \mathcal{R}^{\mathcal{E}}_T - O(\beta D^2 T/\ln^{2} L) 
\end{align*} 
For $L$ sufficiently large (so $O(1/\ln^2 L)$ becomes negligible compared to a constant), the factor 
$$
r ( 1 - e^{-2 \kappa (1 + O(1/\ln^2 L)}) = e^{-\kappa (1 + O(1/\ln^2 L)} ( 1 - e^{-2 \kappa (1 + O(1/\ln^2 L)}) ) (1 -\| \vect{x}_0 \|_{\infty})
$$
attains the maximum value of $\frac{1}{3\sqrt{3}} (1 -\| \vect{x}_0 \|_{\infty})$ at $\kappa = \frac{\ln 3}{2}$.
Hence, we deduce that
\begin{align*}
\sum \limits_{t = 1}^{T} \mathbb{E} \left[F^t(\vect{x}^t_{L+1})\right] 
\geq & \left(\frac{1}{3\sqrt{3}} \right)(1 -\| \vect{x}_0 \|_{\infty}) \sum \limits_{t = 1}^{T} F^t(\vect{x}^*) 
		- O \biggl ( \frac{(\beta D + G + \sigma) D T}{\ln L} \biggr) - O(\mathcal{R}^{\mathcal{E}}_T). 
\end{align*} 
The theorem follows.
\end{proof}

\begin{remark}
\begin{itemize}
\item The regret guarantee, in particular the approximation ratio, depends on the initial solution 
$\vect{x}_{0}$. This confirms the observation that initialization plays an important role in non-convex optimization.
For particular case where $\mathcal{K}$ is the intersection of a cone and the hypercube $[0,1]^{n}$ (so $\vect{0} \in \mathcal{K}$), 
Algorithm \ref{algo:online-pf} provides vanishing $\left(\frac{1}{3\sqrt{3}}\right)$-regret. Note that this is the first constant approximation 
for non-monotone DR-submodular maximization over a non-trivial convex domain beyond the class of down-closed convex domains. 
\item Assume that $\vect{0} \in \mathcal{K}$ and $F^{t}$'s are identical, i.e., $F^{t} = F \forall t$ , 
the algorithm guarantees a convergence rate of $O\left(1/\log T\right)$. 
It means that to be $\epsilon$-close to a solution which is $\frac{1}{3\sqrt{3}}$-approximation to the optimal solution of $F$, 
the algorithm requires $T = O(2^{1/\epsilon})$ iterations. 
Note that the exponential complexity is unavoidable as any contant approximation algorithm for the multilinear extension of a submodular function (so DR-submodular) over a general convex set requires necessarily an 
exponential number of value queries \cite{Vondrak13:Symmetry-and-approximability}.   
%
%
\item The assumptions on the smoothness of functions $F^{t}$'s and the knowledge of $T$ in Algorithm~\ref{algo:online-pf}
can be removed by the argument as Remark \ref{re:down-closed}.
\end{itemize}
\end{remark}

\section{Experiments} 	\label{sec:exp}

In this section, we validate our online algorithms for non-monotone DR submodular optimization on a set of real-world dataset. 
We show the performance of our algorithms for maximizing non-montone DR-submodular function over a down-closed polytope and over a general convex polytope. All experiments  are performed in MATLAB using CPLEX optimization tool on MAC OS version 10.15.

In the revenue maximization problem, the goal of a company is to offer for free or advertise a product to users so that the revenue increases through their ``word-of-mouth" effect on others.
Here, we are given an undirected social network graph $G = (V, W)$, where $w_{ij} \in W$ represents the weight of the edge between vertex $i$ and vertex $j$.  
If the company invests $x_{i}$ unit of cost on an user $i \in V$ then user $i$ becomes an advocate of the product with probability $1- (1-p)^{x_{i}}$ where $p \in (0,1)$ is a parameter. 
Intuitively, this signifies that for investing a unit cost to $i$, we have an extra chance that the user $i$ becomes an advocate with probability $p$~\cite{Soma:2017}.

Let $S \subset V$ be a random set of users who advocate for the product. Then the revenue with respect to $S$ is defined as 
$\sum_{i \in S} \sum_{j \in V \setminus S} w_{ij}$. Let $F : [0,1]^{|E|} \rightarrow \mathbb{R}$ be the expected revenue obtained in this model, that is
\begin{align*}
F(\vect{x}) &= \mathbb{E}_S \bigg[\sum_{i \in S} \sum_{j \in V \setminus S} w_{ij} \bigg] = \sum \limits_{i} \sum \limits_{j: i \not = j} w_{ij} (1-(1-p)^{x_i}) (1-p)^{x_j}
\end{align*}
It has been shown that $F$ is a non-monotone DR-submodular function~\cite{Soma:2017}.


\begin{figure*}[h]
\vspace{-2.5cm}
\centerline {
	\parbox{0.5\textwidth}{\includegraphics [width=0.5\textwidth] {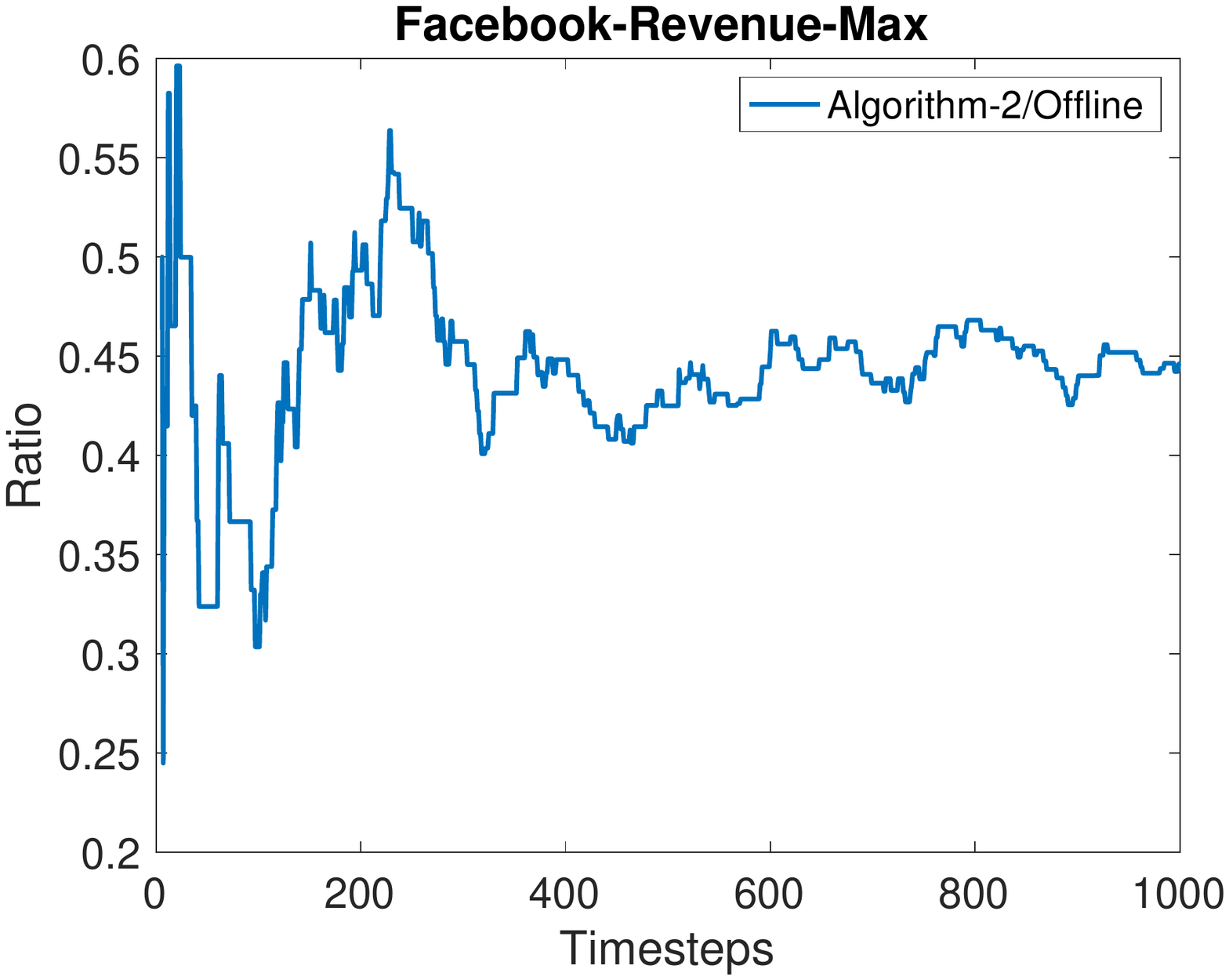}}
	\parbox{0.5\textwidth}{\includegraphics [width=0.5\textwidth] {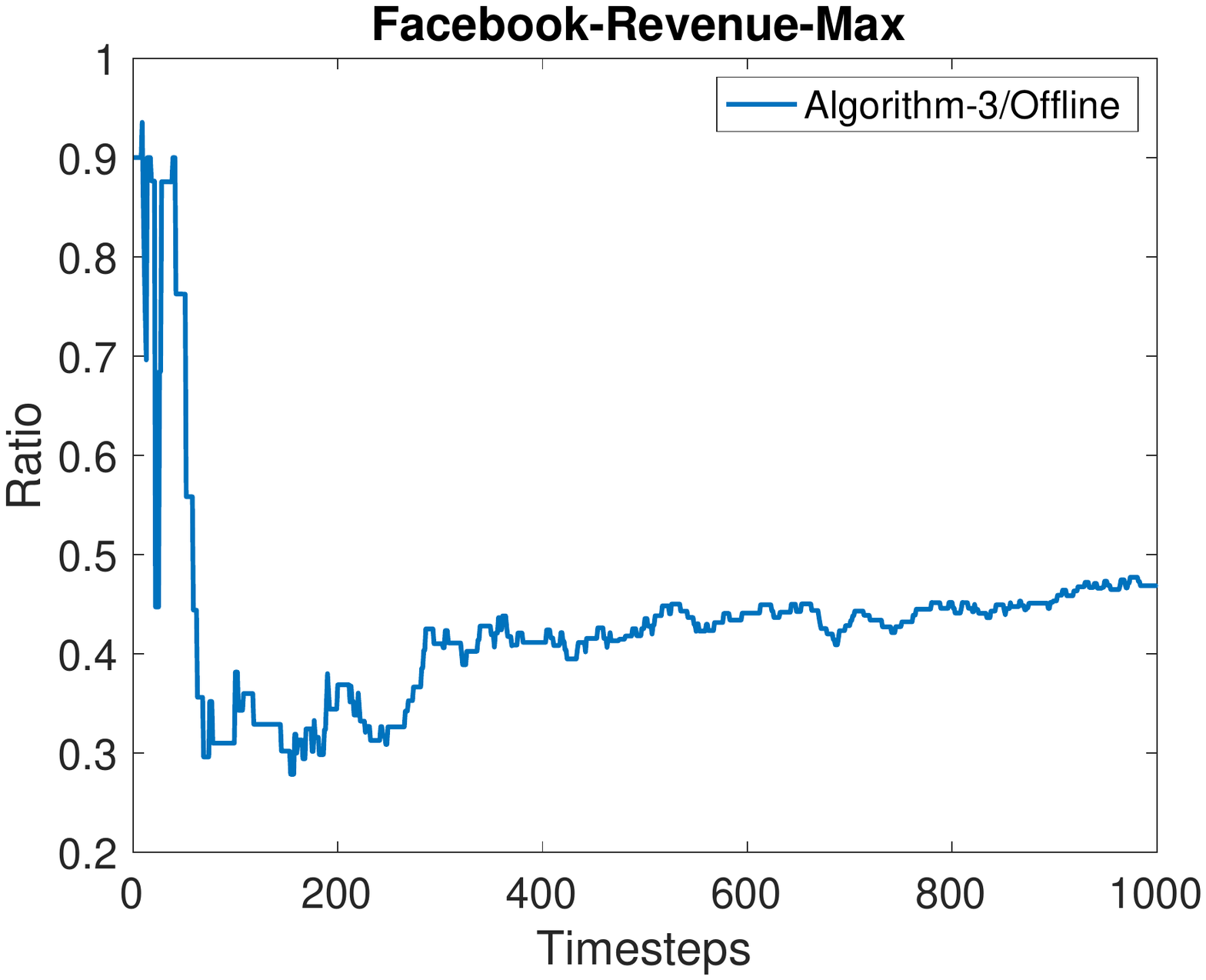}}
	}
\vspace{-2.3cm}
\centerline
{
\parbox{0.5\textwidth}{\centering{(a)}}
\parbox{0.5\textwidth}{\centering{(b)}}
}
\caption{Performance ratio for online revenue maximization over (a) a down-closed polytope; 
and (b) a general polytope.
} 
\label{rw-ratio}
\end{figure*}


%
In our setting, we consider the online variant of the revenue maximization on a social network where at time $t$ the weight of an edge is given $w^t_{ij} \in \{0,1\}$. The experiments are performed on the Facebook dataset that contains $20$K users (vertices) and $1$M relationships (edges). We choose the number of time steps to be $T=1000$. At each time $t \in {1, \ldots, T}$, we randomly uniformly select $2000$ vertices $V^t \subset V$, independently of $V^1,\ldots,V^{t-1}$, and construct a batch $B_t$ with edge-weights $w^t_{ij} = 1$ if and only if $i,j \in V^t$ and edge $(i,j)$ exists in the Facebook dataset. In case if $i$ or $j$ do not belong to $V^t$, $w^t_{ij} = 0$. We set $p = 0.0001$.  In the first experiment, 
we impose a maximum investment constraint on the problem such that $\sum_{i \in V} x_i \leq 1$. This, in addition to $x_i \geq 0, \forall i \in V$ constitutes a \emph{down-closed} feasible convex set. For the general convex set, we impose an additional minimum investment constraint on the problem such that $\sum_{i \in V} x_i \geq 0.1$. 

For comparison purposes, we chose  (offline) Frank-Wolfe algorithm as the benchmark. 
This algorithm is shown to be competitive both in theory and in practice for maximizing (offline) non-monotone DR-submodular functions over down-closed convex sets~\cite{BianLevy17:Non-monotone-Continuous} and over general convex sets \cite{DurrThang20:Non-monotone-DR-submodular}. 
In Figure~\ref{rw-ratio} we show the ratio of between the objective value achieved by the online algorithm and that of the benchmark over the down-closed convex set (Figure~\ref{rw-ratio}(a)) and the general convex set (Figure~\ref{rw-ratio}(b)). 
These experiments conform with the theoretical guarantees. 

\section{Conclusion}

In this paper, we considered the regret minimization problems while maximization online arriving non-monotone DR-submodular functions over convex sets. We presented online algorithms that achieve a $1/e$-approximation and $1/2$-approximation ratios with the regret of $O(T^{2/3})$ and $O(\sqrt{T})$ for maximizing DR-submodular functions over any down-closed convex set and over the hypercube, respectively.  Both of these algorithms are built upon on our novel idea of lifting and solving the optimization problem (the inherent oracles) in higher dimension.  Moreover, we presented an online algorithm that achieves an approximation guarantee (depending on the search space) for the problem of maximizing non-monotone continuous DR-submodular functions over a \emph{general} convex set (not necessarily down-closed). 
Finally,  we run experiments to verify the performance of our algorithms on a social revenue maximization problem on a Facebook user-relationship dataset. Interesting direction for future work is to consider the non-monotone DR-submodular maximization 
in the bandit setting by combining our approach with algorithms for monotone DR-submodular functions in \cite{Zhang:2019}.

%

\bibliographystyle{plain}
\bibliography{submodular}




\newpage
\appendix

\section{Removing Knowledge of $T$}
\label{know-T}
In this section, we show how to remove the assumption on the knowledge of $T$ for Algorithm \ref{algo:online-pf}. 
The procedure for Algorithms \ref{algo:online-FW} is similar. 
In particular, we use the standard doubling trick (for example \cite{ChenZhang19:Projection-Free-Bandit}) where Algorithm~\ref{algo:online-pf} is invoked repeatedly with a doubling time horizon. 

\begin{algorithm}[ht]
\begin{flushleft}
\textbf{Input}:  A convex set $\mathcal{K}$ and Algorithm~\ref{algo:online-pf}
\end{flushleft}
\vspace{-0.25cm}
\begin{algorithmic}[1]
\FOR{$m = 0,1, 2, \ldots$}
\STATE  $L := 2^{m+1}$
\STATE  Run Algorithm~\ref{algo:online-pf} with horizon $2^{m}$, from the $(2^{m}+1)$-th iteration to 
	the $2^{m+1}$-th iteration.
\STATE Let $\vect{x}^{t}$ for $2^{m}+1 \leq t \leq 2^{m+1}$ be the solution fo Algorithm~\ref{algo:online-pf}.
\ENDFOR
\end{algorithmic}
\caption{Meta-Frank-Wolfe with Doubling Trick}
\label{algo:double-trick}
\end{algorithm}

\begin{theorem}
Given Theorem~\ref{thm:online-general}, the following inequality holds true for Algorithm \ref{algo:double-trick}:
\begin{align*}
\sum \limits_{t = 1}^{T} \mathbb{E} \left[F^t(\vect{x}^t)\right] 
\geq & \left( \frac{1}{3\sqrt{3}} \right) (1 -\| \vect{x}_{0} \|_{\infty}) \max_{\vect{x}^* \in {\cal K}} \sum \limits_{t = 1}^{T} F^t(\vect{x}^*) 
		- O \biggl ( \frac{(\beta D + G + \sigma) D T}{\ln T} \biggr)
\end{align*} 
\end{theorem}

\begin{proof}
From Theorem~\ref{thm:online-general}, for each $m$, it follows that 
\begin{align*}
\sum \limits_{t = 2^m+1}^{2^{m+1}} \mathbb{E} \left[F^t(\vect{x}^t)\right] 
\geq & \left( \frac{1}{3\sqrt{3}}  \right) (1 -\| \vect{x}_{0} \|_{\infty}) \max_{\vect{x}^* \in {\cal K}} \sum \limits_{t = 2^m+1}^{2^{m+1}} F^t(\vect{x}^*) 
		- O \biggl ( \frac{(\beta D + G + \sigma) D 2^{m+1}}{(m+1) \ln 2} \biggr). 
\end{align*} 

Summing this quantity of $m = 0, 1, 2, \ldots, \lceil \log_2(T+1) \rceil - 1$, we have that 
\begin{align*}
\sum\limits_{t = 1}^{T} & \mathbb{E} \left[F^t(\vect{x}^t)\right] \\
&\geq  \left(\frac{1}{3\sqrt{3}} \right) (1 -\| \vect{x}_{0} \|_{\infty}) \max_{\vect{x}^* \in {\cal K}} 
\sum \limits_{t = 1}^{T} F^t(\vect{x}^*) 
		- \sum_{m = 0}^{ \lceil \log_2(T+1) \rceil -1}  O \biggl ( \frac{(\beta D + G + \sigma) D 2^{m+1}}{(m+1) \ln 2} \biggr) \\
&\geq  \left( \frac{1}{3\sqrt{3}} \right) (1 -\| \vect{x}_{0} \|_{\infty}) \max_{\vect{x}^* \in {\cal K}} 
\sum \limits_{t = 1}^{T} F^t(\vect{x}^*) 
		- \sum_{m = 0}^{ \lceil \log_2(T+1) \rceil -1}  O \biggl ( \frac{(\beta D + G + \sigma) D 2^{m+1}}{\ln 2} \biggr) \\
&\geq \left( \frac{1}{3\sqrt{3}} \right) (1 -\| \vect{x}_{0} \|_{\infty}) \max_{\vect{x}^* \in {\cal K}} 
\sum \limits_{t = 1}^{T} F^t(\vect{x}^*) 
		-  O \biggl ( \frac{(\beta D + G + \sigma) D}{\ln 2} \biggr) \sum_{m = 1}^{ \lceil \log_2(T+1) \rceil} \left( \frac{2^m}{m} \right)\\
&\geq \left( \frac{1}{3\sqrt{3}} \right)  (1 -\| \vect{x}_{0} \|_{\infty}) \max_{\vect{x}^* \in {\cal K}} 
\sum \limits_{t = 1}^{T} F^t(\vect{x}^*) 
		-  O \biggl ( \frac{(\beta D + G + \sigma) D}{\ln 2} \biggr) \frac{2. 2^{\lceil \log_2(T+1) \rceil}}{\lceil \log_2(T+1) \rceil} \\
&\geq \left( \frac{1}{3\sqrt{3}} \right) (1 -\| \vect{x}_{0} \|_{\infty}) \max_{\vect{x}^* \in {\cal K}} 
\sum \limits_{t = 1}^{T} F^t(\vect{x}^*) 
		-  O \biggl ( \frac{(\beta D + G + \sigma) D T}{\ln T} \biggr).	
\end{align*} 
\end{proof}

\end{document}